\documentclass[msom,nonblindrev]{informs3} % current default for manuscript submission
\OneAndAHalfSpacedXI
\usepackage{natbib}
 \bibpunct[, ]{(}{)}{,}{a}{}{,}%
 \def\bibfont{\small}%
\usepackage{graphicx}
\usepackage{booktabs}
\usepackage{float}
\usepackage{multirow}
\usepackage{subfigure}
\usepackage{hyperref}
\usepackage{algorithm}
\usepackage{threeparttable}
\usepackage{algorithmic}
\usepackage{makecell}
\usepackage{multirow}
\usepackage{amsmath,amssymb}
\usepackage{color}
\usepackage{tabularx}
\usepackage{makecell}
\TheoremsNumberedThrough     % Preferred (Theorem 1, Lemma 1, Theorem 2)
\EquationsNumberedThrough    % Default: (1), (2), ...
\begin{document}
%%%%%%%%%%%%%%%%

%\setcounter{page}{1}
%\DOI{10.1287/xxxx.0000.0000}%


%\RUNTITLE{}

% \TITLE{Bundling Information Goods of Decreasing Value}
\TITLE{Cold-Start Forecasting of New Product Life-Cycles via Conditional Diffusion Models}

\ARTICLEAUTHORS{
\AUTHOR{Ruihan Zhou}
\AFF{Guanghua School of Management, Peking University, Beijing 100871, China, \EMAIL{rhzhou@stu.pku.edu.cn}}

\AUTHOR{Zishi Zhang}
\AFF{Guanghua School of Management, Peking University, Beijing 100871, China, \EMAIL{zishizhang@stu.pku.edu.cn}}

\AUTHOR{Jinhui Han}
\AFF{Guanghua School of Management, Peking University, Beijing 100871, China, \EMAIL{jinhui.han@gsm.pku.edu.cn}}

\AUTHOR{Yijie Peng}
\AFF{Guanghua School of Management, Peking University, Beijing 100871, China, \EMAIL{pengyijie@pku.edu.cn}}

\AUTHOR{Xiaowei Zhang}
\AFF{Department of Industrial Engineering and Decision Analytics, The Hong Kong University of Science and Technology, Hong Kong SAR, China, \EMAIL{xiaoweiz@ust.hk}}
}
%\URL{}}
 % end of the block
\ABSTRACT{\textbf{Problem definition:} Forecasting the life-cycle trajectory of a newly launched product is important for launch planning, resource allocation, and early risk assessment. This task is especially difficult in the pre-launch and early post-launch phases, when product-specific outcome history is limited or unavailable, creating a cold-start problem. In these phases, firms must make decisions before demand patterns become reliably observable, while early signals are often sparse, noisy, and unstable. \textbf{Methodology/results:} We propose the Conditional Diffusion Life-cycle Forecaster (CDLF), a conditional generative framework for forecasting new-product life-cycle trajectories under cold start. CDLF combines three sources of information: static descriptors, reference trajectories from similar products, and newly arriving observations when available. Here, static descriptors refer to structured pre-launch characteristics of the product, such as category, price tier, brand or organization identity, scale, and access conditions. This structure allows the model to condition forecasts on relevant product context and to update them adaptively over time without retraining, yielding flexible multi-modal predictive distributions under extreme data scarcity. The method satisfies consistency with a horizon-uniform distributional error bound for recursive generation. Across studies on Intel microprocessor stock keeping unit (SKU) life cycles and the platform-mediated adoption of open large language model repositories, CDLF delivers more accurate point forecasts and higher-quality probabilistic forecasts than classical diffusion models, Bayesian updating approaches, and other state-of-the-art machine-learning baselines. \textbf{Managerial implications:} Our results show how generative forecasting can support managerial judgment in cold-start settings where conventional forecasting methods are least reliable. By producing full predictive distributions and updating them as new information arrives, CDLF helps managers assess launch-stage upside potential, downside risk, and forecast uncertainty in a timely and disciplined way. More broadly, the framework provides a foundation for decision-relevant forecasting in new product launches and other operational settings characterized by limited history and rapidly evolving information.}

%\KEYWORDS{deterministic inventory theory; infinite linear programming duality; 
%  existence of optimal policies; semi-Markov decision process; cyclic schedule}

\maketitle

\section{Introduction}\label{section1}

Forecasting the future life-cycle of a newly launched product, often referred to as the \emph{cold-start} forecasting problem, is important in operations and marketing because consequential decisions must be made \emph{before} the product's own life-cycle pattern becomes observable \citep{Lei2023,Ban2018,hu2019}. In this setting, the goal is not merely to predict the next few observations from a short history, but to infer how the \emph{entire life-cycle} of a new product may evolve over time: whether adoption will rise gradually, peak early, decline quickly, or remain persistent. Figure~\ref{fig:lifecyclepipeline}, adapted from \cite{Lei2023}, summarizes the launch-stage information phases that motivate our study. In the \textit{Pre-Launch} phase, no outcome history is available, so prediction must rely on static product descriptors and analog information from related products. In the \textit{Early Post-Launch} phase, only a short and often noisy prefix is observed, which is usually insufficient to reveal the underlying life-cycle shape. Only in the \textit{Late Post-Launch} phase does the product's own history become informative enough for conventional time-series methods to extrapolate from realized observations. The central difficulty, therefore, is not simply limited data volume, but the absence of product-specific dynamic information at the very stage when managers need a credible forecast of the full life-cycle. The central challenge of cold-start forecasting has several distinct but related layers. First, in the pre-launch and early post-launch phases, product-specific observations are absent or extremely limited, making it difficult to forecast the initial stages of the life cycle with confidence. Second, even beyond the earliest stages, the evolution of the full life-cycle trajectory is difficult to infer from the focal product's own observed history alone. Life-cycle dynamics are path-dependent, stage-specific, and heterogeneous across products, so forecasting cannot be reduced to a standard extrapolation problem based only on within-product observations. Third, the task is inherently dynamic and distributional: the model must not only generate a plausible predictive distribution over the future trajectory under severe scarcity, but also update that distribution as new information arrives. Taken together, these challenges motivate the need for cross-product transfer through analog information and static descriptors, combined with adaptive updating as product-specific evidence gradually accumulates. More broadly, recent operations work also shows that transfer learning across related signals and aggregation levels can materially improve demand prediction when product-level information is limited \citep{Lei2024Transfer}.

\begin{figure}[htbp]
    \centering
    \includegraphics[width=0.8\linewidth]{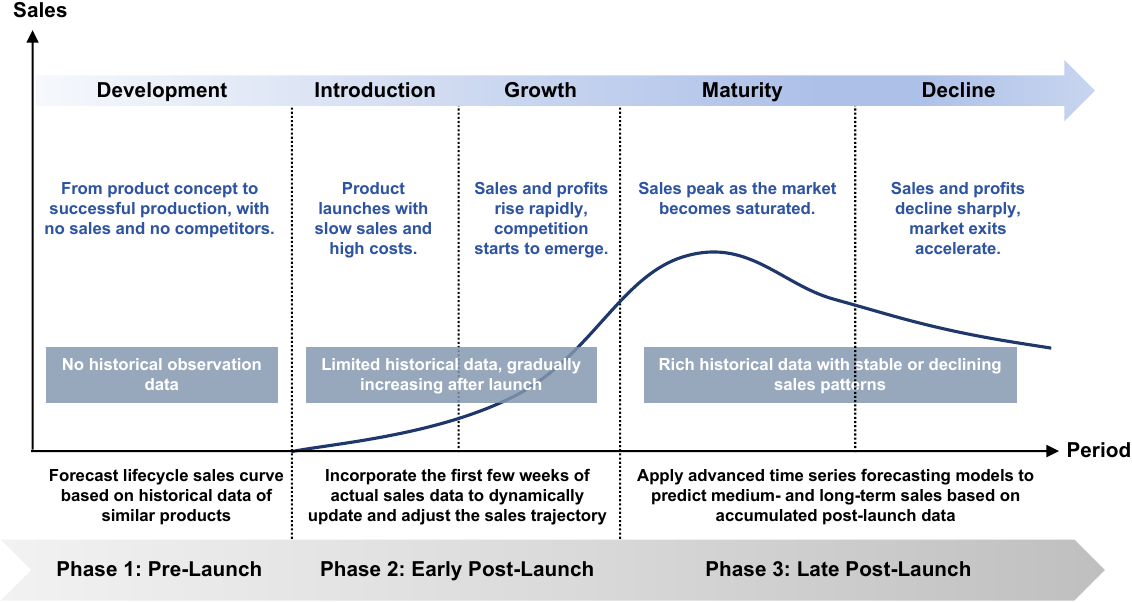}
    \caption{The pipeline of product sales life-cycle forecasting and the associated information phases.}
    \label{fig:lifecyclepipeline}
\end{figure}

An important and increasingly common setting for cold-start forecasting is the release of \emph{platform-mediated AI products}, such as open-source large language model (LLM) repositories, model checkpoints, developer toolkits on digital platforms, and related artifacts whose early adoption is often reflected in public engagement signals such as GitHub star counts. These products spread through platform discovery channels such as search, recommendation, and trending lists, which amplify exposure and drive adoption. Unlike traditional products, platform-mediated releases often exhibit adoption dynamics shaped by external factors such as ranking and recommendation systems. These systems can generate feedback loops that influence both the level and the timing of adoption. Adoption outcomes are typically skewed, with a small fraction of releases attracting the majority of attention, while most remain niche. As a result, forecasting these products requires handling unpredictable dynamics from the start, especially when only limited data is available. The main challenge for managers is not just improving short-term predictions, but also supporting \emph{early planning} with minimal evidence. Decisions such as how much traffic-acquisition effort to allocate, whether to prepare evaluation or hosting capacity for potential surges in adoption, and how aggressively to sequence promotional or ecosystem-support actions must be made with substantial uncertainty about how the product’s life-cycle will evolve. To address these challenges, we need a forecasting framework that can handle sparse early data, adapt to evolving information, and provide flexible and coherent predictions of the product's future life-cycle. 

\subsection{From pre-specified forecasts to predictive distributions over product life-cycles}\label{sec:intro_gap}

Existing approaches to new-product forecasting provide certain ways to make predictions when the target product has little or no history. They do so by introducing structure, borrowing information from similar historical products, or combining both sources of information.These approaches are valuable because they enable forecasting in cold-start settings and often retain an interpretable structure that is useful for analysis and decision support. At the same time, many existing approaches are less well suited to the setting we study. In cold-start environments, the forecast is often partly determined in advance, either by a restricted family of life-cycle shapes or by patterns drawn from similar historical products. This can be useful when the target product follows familiar launch patterns. However, it is more limiting when the product is novel and its life-cycle depends on information that is revealed only gradually after launch. In particular, when only a short and noisy prefix is available, such approaches may not adapt flexibly enough to how the specific product's life-cycle is beginning to evolve.

This limitation is particularly consequential in platform-mediated release settings, where forecast quality depends on the ability to characterize not only expected adoption, but also uncertainty over the future trajectory. It is to characterize the range of plausible life-cycles that may emerge for a newly released product. For example, two products with similar early signals may still differ substantially in whether adoption peaks early or late, declines quickly or persists, and remains modest or grows to an unusually large scale. Accordingly, this paper asks and try to solve the key problem:
\emph{How can we construct a unified cold-start forecasting framework that generates predictive distributions over full product life-cycles, while adapting to evolving information from no observations to short prefixes?}

\subsection{A unified recursive generative forecaster for cold-start life-cycle prediction}\label{sec:intro_method_contrib}

We propose the \emph{Conditional Diffusion Life-cycle Forecaster (CDLF)}, a unified forecasting framework for cold-start product launches. Here, diffusion refers to \emph{diffusion probabilistic models} in machine learning, rather than to adoption-based diffusion models such as Bass. CDLF generates a \emph{predictive distribution} over the full future life-cycle in the two most important information phases in Figure~\ref{fig:lifecyclepipeline}. Methodologically, CDLF extends diffusion-based probabilistic forecasting to a cold-start setting in which the target product has little or no within-product history when forecasts are required \citep{Rasul2021,Shen2023}. The model conditions future life-cycle generation on three information sources: static product descriptors, a reference set of related historical life-cycles, and a short observed prefix when available. Static product descriptors are encoded into a fixed representation of the target product's cross-sectional characteristics. The reference set is used to construct a transferable summary of plausible life-cycle patterns from related historical products. When early observations become available, the observed prefix is encoded as a product-specific dynamic state. These components are then combined to form the conditioning context for recursive generation of the future life-cycle.

This conditioning structure serves two purposes. First, in the \textit{Pre-Launch} phase, when no within-product history is available, forecasts can still be anchored by static product descriptors and cross-entity information from analogous products. Second, in the \textit{Early Post-Launch} phase, the same framework updates the predictive distribution by incorporating the realized prefix of the target product as product-specific information begins to emerge. In this way, the model does not simply interpolate among historical patterns; it revises the forecast as new evidence is revealed and propagates that information through the remaining life-cycle. Unlike approaches that either impose a fixed family of life-cycle shapes or update forecasts through separate retrieval and blending steps, CDLF combines cross-entity transfer and product-specific updating within a single conditional generative framework.

This feature is especially important because many quantities of interest depend on the joint path distribution of the life-cycle over time, including cumulative adoption, peak magnitude, peak timing, and the probability that adoption exceeds a service threshold. Such quantities cannot be characterized well by point forecasts at individual horizons, because they depend on how future outcomes co-move across time. By generating coherent samples of complete future life-cycles conditional on the current information set, CDLF represents both cross-product regularities and product-specific updating within a single probabilistic forecasting framework. We do not attempt to causally identify platform exposure effects such as ranking or featuring. Instead, we treat platform-mediated shocks as part of the uncertainty to be forecast and use the resulting predictive distribution to summarize tail-sensitive launch risk.

Beyond methodology, we provide theoretical guarantees for recursive conditional generation, establishing horizon-uniform control of multi-step distributional error under a stability margin condition. Empirically, we evaluate CDLF in two complementary studies. The first uses a canonical benchmark on physical-product life-cycles and allows comparison with established forecasting baselines in a familiar setting. The second studies the adoption of LLM repositories on a digital platform, where outcomes are highly skewed, exposure-driven, and often bursty. This setting is substantively important in its own right, and it also provides a demanding test bed for cold-start forecasting because products with similar early signals may later exhibit very different life-cycle patterns. Across both studies, we compare CDLF with classical diffusion models, analog-based approaches, and modern probabilistic time-series baselines. We show that CDLF improves both point and distributional forecast performance, while providing predictive life-cycle samples that better capture peak-related and tail-sensitive features of launch dynamics.

The remainder of this paper is organized as follows. Section~\ref{secreview} reviews related work. Section~\ref{sec2} formalizes cold-start life-cycle forecasting and the evolving information phases in Figure~\ref{fig:lifecyclepipeline}. Section~\ref{sec3} presents CDLF, its recursive conditional generation procedure, and the theoretical analysis. Section~\ref{sec5} reports empirical results and benchmark comparisons in two studies, emphasizing both distributional accuracy and event-level life-cycle summaries relevant to launch planning. Section~\ref{sec6} concludes and discusses implications and future directions.

\section{Related Work}\label{secreview}

Our work lies at the intersection of three literature streams: cold-start new-product forecasting, probabilistic and generative models for time-series forecasting, and research on platform-mediated exposure dynamics with heavy-tailed and bursty adoption outcomes. We briefly review each stream and highlight the methodological gap addressed by CDLF.

\subsection{Cold-start New Product Forecasting and Diffusion Models}

New-product demand forecasting is central in marketing and operations because inventory, capacity, and marketing decisions must often be made with little or no product-specific history \citep{hu2019,Lei2023}. More broadly, operations research has also studied learning-and-control problems in which demand information is incomplete or only partially observed, so firms must update operational decisions from limited sales data rather than from a fully observed demand process \citep{chen2021nonparametric}. A classic approach in cold-start forecasting is provided by diffusion models, exemplified by the Bass model, which link adoption to innovation and imitation forces in an interpretable way \citep{Bass1969}. This literature has been extended to accommodate heterogeneity, competition, and richer market environments \citep{1990New,PeresMullerMahajan2010DiffusionReviewIJRM,2002An,Bemmaor2002}. These models are valuable as structural descriptions of adoption dynamics, but for cold-start forecasting they typically rely on restrictive functional-form assumptions and careful early-stage calibration. As a result, they are most effective when the future life-cycle can be represented reasonably well by a pre-specified family of adoption curves. This can be limiting when launch outcomes are highly skewed, shaped by platform exposure, or subject to abrupt changes in timing and scale.

A complementary stream of research develops data-driven methods for forecasting new product life-cycles by borrowing information across products. \cite{hu2019} develop a practical framework for forecasting new product life-cycle curves from historical data. \cite{Lei2023} propose a Bayesian functional approach that incorporates product attributes and promotion information into life-cycle forecasting. \cite{Guo2025} develop a Bayesian ensemble approach for forecasting the full demand distribution of new products across launch stages. Related operations work also studies transfer learning for demand prediction by pooling information across related signals and aggregation levels, highlighting a broader route for cross-entity information sharing under data scarcity \citep{Lei2024Transfer}. More broadly, these approaches use analog products, covariates, pooling, or ensemble learning to anchor forecasts when the target product has limited history. Their key advantage is cross-entity transfer: information from related products can still inform the forecast when the target product has little history. However, the resulting forecast often remains tied to previously observed life-cycle patterns or a restricted family of admissible shapes. When early observations arrive, updates are commonly implemented through matching, weighting, pooling, or other modular adjustments, rather than through a unified mechanism that updates the full predictive distribution of the future life-cycle. This limitation is especially important when short prefixes are noisy and later outcomes may diverge sharply despite similar starts.

Our work shares the goal of leveraging static product descriptors, reference life-cycles, and short observed prefixes when available, but differs in both forecasting object and updating mechanism. Rather than producing a point forecast, an expected life-cycle, or a weighted combination of candidate curves, we target a \emph{conditional predictive distribution} over the full future life-cycle. This distribution is generated within a single recursive conditional framework that incorporates cross-product information before launch and updates as product-specific evidence emerges after launch. This design is especially useful under severe data scarcity because it represents uncertainty about how the life-cycle may unfold, rather than restricting the forecast to a fixed functional form or a finite library of previously observed patterns.

\subsection{Probabilistic and Generative Models for Time-series Forecasting}

Recent work in probabilistic time-series forecasting increasingly uses generative models to represent predictive distributions beyond simple parametric noise assumptions. It is important to distinguish these models from the diffusion models discussed in the marketing literature above. In marketing, diffusion models typically refer to adoption models, such as the Bass framework, that describe how demand spreads over time. By contrast, in machine learning, diffusion probabilistic models are generative models that represent complex predictive distributions through iterative denoising. TimeGrad \citep{Rasul2021} introduced an autoregressive diffusion-based forecaster and demonstrated improved uncertainty calibration, and subsequent studies have adapted diffusion architectures to alternative temporal structures and data sparsity settings \citep{yan2021,Shen2023}. Related conditional generative approaches, including GAN- and VAE-based models, have also been studied for stochastic simulation and uncertainty-aware forecasting \citep{Yoon2019,cont2022tail}. In quantitative finance, diffusion and score-based models have been used for scenario generation in heavy-tailed time series and structured objects such as implied-volatility surfaces \citep{TakahashiMizuno2025SyntheticFTS,ChenXuXuZhang2025DiffusionFactorModels,JinAgarwal2025IVDiffusion}. However, these settings typically assume mature histories for a fixed set of assets and do not address cross-entity cold-start forecasting.

The gap, therefore, is not probabilistic forecasting per se, but probabilistic forecasting under cross-entity cold-start. In our setting, the target product may have no observed history at launch and only a short, noisy prefix shortly thereafter. This requires a forecasting model that conditions on static product descriptors and reference life-cycles from related products, while updating coherently as product-specific evidence accumulates. Existing methods satisfy only part of this requirement. Analog-based and cross-product pooling approaches can transfer information across products, but they generally do not provide a unified probabilistic mechanism for generating and updating coherent future life-cycles. Diffusion forecasters can generate coherent samples of future trajectories, but they are typically developed for settings with substantial within-entity history. Table~\ref{tab:req_vs_baselines} summarizes this gap in terms of four design requirements for cold-start life-cycle forecasting.

\begin{table}[htp]
\centering
\small
\caption{Baseline families and cold-start design requirements.
R1: no within-product history (Pre-Launch);
R2: conditions on analog life-cycles and static product descriptors;
R3: generates coherent life-cycle samples (supporting peak timing, peak magnitude, and cumulative adoption);
R4: uses a single model that updates as a prefix arrives.}
\label{tab:req_vs_baselines}
\begin{tabular}{lcccc}
\hline
Method family & R1 & R2 & R3 & R4 \\
\hline
Bass / parametric diffusion \citep{Bass1969,PeresMullerMahajan2010DiffusionReviewIJRM} & Limited & Limited & No & Limited \\
\makecell[l]{Analog-based / Bayesian ensemble methods 
\\ \citep{hu2019,Lei2023,Guo2025}} & Yes & Yes & Limited & Limited \\
Diffusion forecasters \citep{Rasul2021,Shen2023} & No & No & Yes & No \\
\textbf{CDLF (ours)} & \textbf{Yes} & \textbf{Yes} & \textbf{Yes} & \textbf{Yes} \\
\hline
\end{tabular}
\vspace{0.25em}

\begin{minipage}{0.95\linewidth}
\footnotesize
\textit{Notes.} ``Limited'' indicates partial support, for example through strong functional-form assumptions, reliance on finite libraries of historical life-cycle patterns, modular updating procedures, or limited adaptability across information phases.
\end{minipage}
\end{table}

Motivated by this gap, our framework conditions future life-cycle generation on three information sources: static product descriptors, a reference set of analog life-cycles, and an evolving early prefix when available. This design combines cross-product transfer and sequential updating within a single probabilistic framework, producing coherent samples of full future life-cycles under both Pre-Launch and Early Post-Launch information phases. More broadly, under information scarcity, uncertainty can be handled through either flexible generative prediction or Bayesian posterior-based updating; related methodological work also studies posterior-based sequential rules for uncertainty-aware evaluation, although not in a forecasting setting per se \citep{Rasul2021,Shen2023,Lei2023,Eckman2022Posterior}.
\subsection{Platform-mediated Exposure Dynamics and Heavy-tailed, Bursty Adoption Outcomes}
In platform-mediated digital releases, adoption is shaped not only by product characteristics but also by platform discovery and exposure processes. Evidence from digital channels shows that platform-side interventions can materially affect users' adoption decisions and subsequent usage outcomes, reinforcing the role of the surrounding information environment in shaping early adoption trajectories \citep{sun2019motivating}. Ranking and recommendation systems can generate feedback effects that influence what users see and subsequently adopt \citep{FlederHosanagar2009RecommendersDiversity}, while online diffusion dynamics often exhibit bursty responses and heterogeneous decay patterns consistent with discrete exposure shocks \citep{CraneSornette2008ResponseFunction}. Related work further shows that diffusion scale and diffusion structure need not coincide \citep{GoelAndersonHofmanWatts2016StructuralVirality}. Together, these findings suggest that the adoption life-cycle of a newly released digital product may vary not only in scale, but also in peak timing, persistence, and the speed of takeoff and decline.

These exposure mechanisms interact with highly uneven adoption outcomes. In scalable digital markets, a small fraction of releases often accounts for a disproportionate share of total adoption, consistent with superstar and long-tail patterns \citep{Rosen1981Superstars,Anderson2006LongTail,BrynjolfssonHuSmith2003ConsumerSurplus}. This feature matters for forecasting because heavy-tailed adoption outcomes are often poorly summarized by a single expected life-cycle or by models built around a restricted family of smooth adoption curves. Delayed takeoff, abrupt surges, unusually large peaks, and persistently high adoption may therefore be difficult to capture through point forecasts or finite libraries of representative historical patterns.

These considerations reinforce the need for a forecasting framework that models the full predictive distribution of the future adoption life-cycle rather than only its conditional mean. This need is especially pronounced in cold-start settings, where the forecast must initially rely on cross-product regularities and then update as short observed prefixes begin to reveal whether the target product is following an ordinary or an extreme adoption pattern. A generative forecasting framework is therefore attractive because it can represent flexible, non-Gaussian, and potentially multimodal uncertainty over future adoption life-cycles while producing coherent life-cycle samples under evolving information sets.

\section{Problem Formulation}\label{sec2}

This study considers the problem of cold-start forecasting for product life-cycle trajectories.
A product is in a cold-start phase when little or no product-specific outcome history is available.
This setting is common at launch.
Before release, no within-product observations exist.
Shortly after release, only a brief and often noisy prefix is observed.
As a result, conventional time-series methods that rely mainly on long product-specific histories are not well suited to this setting.

The difficulty of cold-start forecasting has two distinct aspects.
The first is \emph{pre-launch forecasting}.
At this stage, the goal is to predict the initial and subsequent evolution of a new product when no product-specific observations have yet been collected.
The second is \emph{early post-launch forecasting}.
At this stage, a short prefix of the focal product becomes available.
The goal is then to update the forecast by combining this emerging product-specific evidence with information learned from other products.
Therefore, cold-start forecasting is not only a problem of prediction under missing history, but also a problem of progressive adaptation as product-specific information accumulates.

Let $\mathbf{x}_t \in \mathbb{R}^D$ denote the vector of observed outcomes at time $t$, where $D$ is the number of outcome dimensions.
The full life-cycle trajectory over a horizon of length $T$ is $\mathbf{x}_{1:T} = (\mathbf{x}_1,\mathbf{x}_2,\dots,\mathbf{x}_T) \in \mathbb{R}^{D \times T}$. Equivalently, if $x_{i,t}\in\mathbb{R}$ denotes the $i$th outcome component at time $t$ for $i \in \{1,\dots,D\}$, then $\mathbf{x}_t = (x_{1,t},x_{2,t},\dots,x_{D,t})$. Our objective is to model the distribution of the life-cycle trajectory $\mathbf{x}_{1:T}$. For each focal product, we observe a vector of static attributes $\mathbf{s}\in\mathbb{R}^p$. These attributes describe time-invariant characteristics of the product and its launch environment, such as category, price tier, brand, or other pre-launch attributes after preprocessing and encoding.
In this paper, $\mathbf{s}$ is static over the forecasting horizon. It is not a dynamic state variable.
Cold-start forecasting is inherently cross-product.
To support inference for a new product, we use a reference library
\[
\mathcal{X}
=
\left\{
\left(\mathbf{x}^{(1)}_{1:T},\mathbf{s}^{(1)}\right),
\dots,
\left(\mathbf{x}^{(K)}_{1:T},\mathbf{s}^{(K)}\right)
\right\},
\]
where each $\mathbf{x}^{(k)}_{1:T}\in\mathbb{R}^{D\times T}$ denotes the complete life-cycle trajectory of a historical product, and each $\mathbf{s}^{(k)}\in\mathbb{R}^p$ denotes its static attributes.
The reference products are selected to be similar to the focal product based on these static characteristics.
The detailed construction of the reference set and the corresponding similarity measure are presented in Section~\ref{sec31}.

The pre-launch information available for a focal product therefore has two components.
The first is the static attribute vector $\mathbf{s}$.
The second is the historical reference set $\mathcal{X}$.
Throughout this paper, we use the term \emph{contextual information} to refer to this pre-launch information.
Thus, contextual information does not replace product-specific data.
Instead, it provides the information available before such data exist, and it remains useful later when forecasts are updated using the focal product's own observed prefix.

To formalize the information available over time, we define a filtration $\{\mathcal{F}_t\}_{t=0}^T$ by
\[
\mathcal{F}_0 = \sigma(\mathcal{X},\mathbf{s}),
\qquad
\mathcal{F}_t = \sigma(\mathbf{x}_{1:t},\mathcal{X},\mathbf{s}),\quad t=1,2,\dots,T.
\]
Here, $\mathcal{F}_0$ represents the pre-launch information set, when no product-specific outcomes have been observed.
For $t \ge 1$, the information set expands to include the observed prefix $\mathbf{x}_{1:t}$ of the focal product.
By construction,
\[
\mathcal{F}_0 \subseteq \mathcal{F}_1 \subseteq \cdots \subseteq \mathcal{F}_T.
\]
The forecasting target is the conditional distribution of the future trajectory given the information available up to the forecast origin.
For any $t_0\in\{1,\dots,T\}$, we aim to learn
\begin{equation}\label{eq:target_future_conditional}
q(\mathbf{x}_{t_0:T}\mid \mathcal{F}_{t_0-1})
=
\prod_{t=t_0}^T q(\mathbf{x}_t\mid \mathcal{F}_{t-1}).
\end{equation}

This formulation unifies the two cold-start stages. In the \emph{pre-launch} stage, $t_0=1$ and the model must forecast the full life-cycle trajectory without any product-specific observations.
The corresponding target distribution is 
$q(\mathbf{x}_{1:T}\mid \mathcal{F}_0)
=
q(\mathbf{x}_1\mid \mathcal{F}_0)\prod_{t=2}^T q(\mathbf{x}_t\mid \mathbf{x}_{1:t-1},\mathcal{F}_0)$.
At this stage, prediction relies entirely on static attributes of the focal product and life-cycle patterns learned from comparable historical products. In the \emph{early post-launch} stage, $t_0>1$ and a short prefix $\mathbf{x}_{1:t_0-1}$ has become available.
The forecasting problem then changes.
The model should no longer rely only on cross-product regularities.
It should also adapt to the focal product by incorporating its observed prefix into the conditioning information.
In this sense, early post-launch forecasting is an adaptation problem: the model uses shared patterns learned from historical products, while refining the forecast using the focal product's own emerging trajectory.

This distinction is central to our formulation.
Before launch, no product-specific evidence exists, so inference must be anchored by $\mathbf{s}$ and $\mathcal{X}$.
After launch, product-specific evidence gradually accumulates, and the forecast should be updated accordingly.
A useful cold-start forecasting model must therefore perform well in both phases: it must generate reasonable life-cycle forecasts when no within-product history is available, and it must adapt those forecasts as limited product-specific observations arrive. In this work, we focus on the core forecasting target, namely the directly observed outcome trajectory $\mathbf{x}_{1:T}$.
We do not incorporate additional auxiliary behavioral signals.
The next section introduces our conditional generative model for approximating the distributions in \eqref{eq:target_future_conditional}.

\section{Methodology}\label{sec3}

This section presents our methodology for cold-start life-cycle forecasting. We begin with an overview of the proposed framework and highlight the problem-specific design choices that address the pre-launch and early post-launch phases. We then describe how contextual information is encoded and summarized into an evolving latent state, and finally introduce the conditional diffusion model and the joint training procedure. The theoretical properties of the proposed framework are analyzed in Section~\ref{sec:theory_main_bilingual}, where we establish forecasting error bounds under suitable stability conditions.
\subsection{Method Overview and Problem-specific Design}\label{sec30}

In this section, we formulate cold-start forecasting as a conditional generative problem. Our objective is to model the conditional distribution of the remaining life-cycle trajectory given the information available at the forecast origin. Formally, for a forecast origin $t_0$, the target is the conditional distribution of future outcomes from $t_0$ to $T$. This formulation unifies the two cold-start phases considered in this paper. To address both phases within a single framework, we develop a unified conditional generative model in which future observations are generated one step at a time. At each step, the predictive distribution is represented by a conditional denoising diffusion model.

The key difficulty of cold-start forecasting is not only uncertainty in future outcomes, but also the limited and evolving nature of the information available for prediction. In the pre-launch phase, no product-specific outcome history is available, so forecasting must rely on the static descriptors of the focal product and information borrowed from comparable historical products. In the early post-launch phase, a short observed prefix becomes available, and the model must adapt its forecasts to the focal product while still leveraging cross-product regularities. Our framework is designed to address these two phases within a single model.

The core of our method is therefore not diffusion per se, but a unified conditional generative framework with an evolving latent information state. This design has three components. First, we encode pre-launch contextual information, including the static descriptors of the focal product and reference trajectories from similar historical products, into compact embeddings. Second, we use an autoregressive state-transition mechanism to update the latent state as product-specific observations accumulate. Third, conditioned on the current latent state, we use a denoising diffusion model to generate the next-period outcome vector. These three components together allow the model to support cross-product transfer under no history, product-specific adaptation once a short prefix becomes available, and distributional forecasting throughout the life-cycle.

This temporal setting differs fundamentally from standard diffusion models developed for image generation. In image generation, dependencies are primarily spatial. In forecasting, by contrast, observations are ordered over time, and predictions must be consistent with this ordering. This distinction is especially important in cold-start settings, where little or no product-specific outcome history is available. Figure~\ref{fig:methodcompare} illustrates the difference between standard diffusion models for images and our proposed CDLF framework for cold-start product life-cycle forecasting. Unlike approaches that condition only on an observed time-series prefix, CDLF is designed to incorporate a richer information set. This information includes static descriptors of the focal product, reference trajectories from similar historical products, and, when available, a short observed prefix of the focal product. This temporal setting differs fundamentally from standard diffusion models developed for image generation. In image generation, dependencies are primarily spatial. In forecasting, by contrast, observations are ordered over time, and predictions must be consistent with this ordering. This distinction is especially important in cold-start settings, where little or no product-specific outcome history is available. We next describe how this information is encoded and organized into the latent state that drives the forecasting process.
\begin{figure}[htp]
    \centering
    \includegraphics[width=0.6\linewidth]{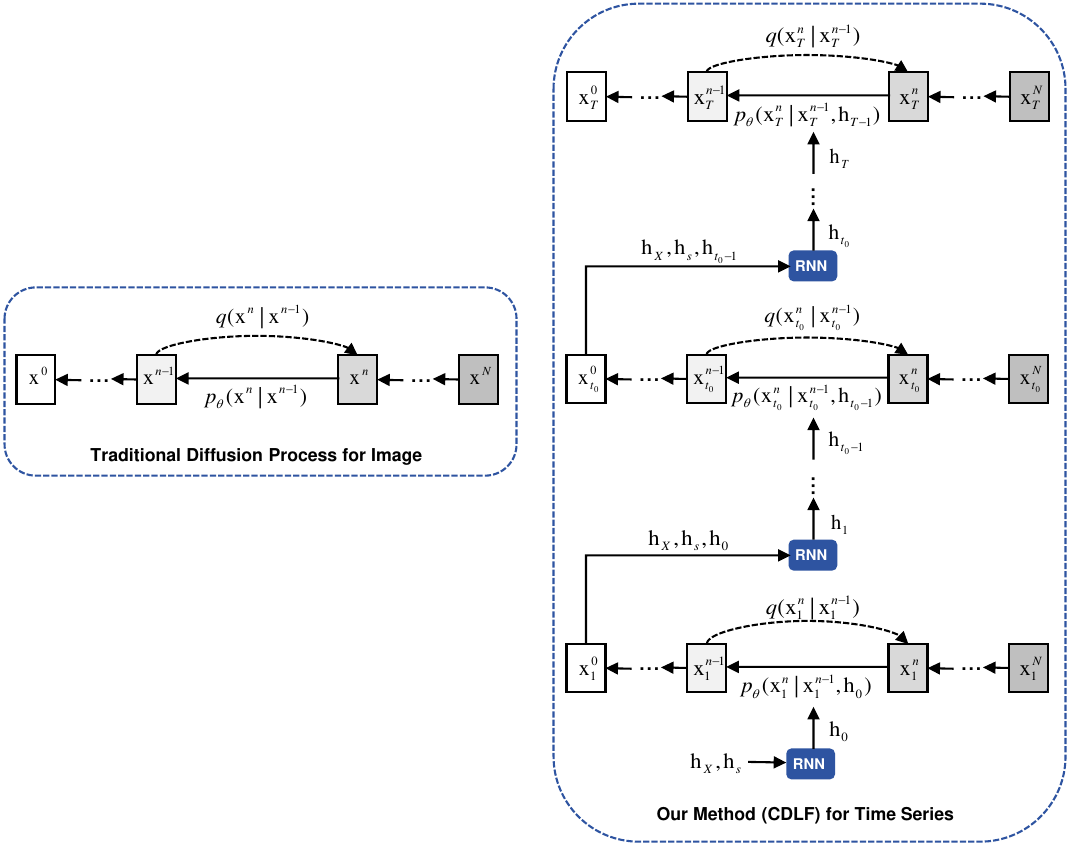}
    \caption{Comparison Schematic of Diffusion Modeling Paradigms for Images and Time Series Forecasting}
    \label{fig:methodcompare}
\end{figure}

\subsection{Contextual Information Encoding and State Construction}\label{sec31}

We now describe how the information introduced in Section~\ref{sec2} is encoded and incorporated into the forecasting model. Let $\mathbf{s}\in\mathbb{R}^p$ denote the static descriptors of the focal product, and let $\mathcal{X}=\{(\mathbf{x}^{(k)}_{1:T}, \mathbf{s}^{(k)})\}_{k=1}^{K}$ denote a reference library of comparable historical product trajectories. Each reference product is also associated with its own static descriptors $\mathbf{s}^{(k)}$. We encode this information into two fixed contextual embeddings, $\mathbf{h}_{\mathcal X}$ and $\mathbf{h}_s$, and then use a recurrent hidden state $\mathbf{h}_t$ to summarize the information available up to time $t$.

This construction is motivated directly by the cold-start setting. In the pre-launch phase, static descriptors matter because they provide product-level information before any product-specific observations are available. Reference products matter because, under no history, cross-product regularities are the main source of information for anticipating the likely shape of the focal product's life-cycle. Once an observed prefix becomes available, a recurrent state is needed to summarize the evolving product-specific information and to update the forecast accordingly. Our state construction is designed to combine these three roles within a single representation.

We first construct the reference set by comparing the static descriptors of the focal product with those of candidate historical products. Specifically, for each candidate product $k$, we compute the similarity score $\mathrm{sim}(k)=-\|\mathbf{s}^{(k)}-\mathbf{s}\|_2^2$. We then select the top $K$ historical products with the highest similarity scores, or equivalently, the smallest Euclidean distances in the descriptor space.

For each selected reference product, we encode its life-cycle trajectory $\mathbf{x}^{(k)}_{1:T}$ using a Recurrent Neural Network (RNN) encoder, which maps the trajectory to a fixed-dimensional representation $\mathbf{h}_k$. This encoder is suitable because historical trajectories may differ in effective length and temporal pattern, while the downstream forecasting model requires a common latent representation.

To reflect the relevance of different reference products, we assign soft importance weights through a softmax transformation:
\[
\omega_k=
\frac{\exp\!\left(-\gamma\|\mathbf{s}^{(k)}-\mathbf{s}\|_2^2\right)}
{\sum_{j=1}^{K}\exp\!\left(-\gamma\|\mathbf{s}^{(j)}-\mathbf{s}\|_2^2\right)}.
\]
Rather than using $\omega_k$ as a multiplicative scalar on $\mathbf{h}_k$, we concatenate the weight with the reference representation and then apply a learnable projection:
\[
\tilde{\mathbf{h}}_k
=
\mathrm{ReLU}\!\left(
\mathbf{W}_{\omega}[\mathbf{h}_k;\omega_k]+\mathbf{b}_{\omega}
\right).
\]
We then aggregate the resulting reference representations through
\[
\mathbf{h}_{\mathcal{X}}
=
\mathrm{RNN}_{\mathrm{ref}}
\!\left(
\{\tilde{\mathbf{h}}_1,\tilde{\mathbf{h}}_2,\dots,\tilde{\mathbf{h}}_K\}
\right).
\]
Here, $[\cdot;\cdot]$ denotes concatenation, and the projection matrix $\mathbf{W}_{\omega}$ maps the $(d+1)$-dimensional concatenated vector back to $d$ dimensions so that the interface with the downstream aggregator remains unchanged. When applying $\mathrm{RNN}_{\mathrm{ref}}(\cdot)$, we order the reference products by $\omega_k$ in descending order. This design allows the model to learn a flexible nonlinear interaction between similarity information and reference content, rather than forcing similarity to act only through direct multiplicative scaling. As shown in Appendix~\ref{app:weight_fusion_2ways}, this concatenation-based fusion performs better than the alternative multiplicative weighting scheme.

We next encode the static descriptors $\mathbf{s}$ of the focal product into a fixed-dimensional embedding $\mathbf{h}_s$ through a feedforward neural network. Both $\mathbf{h}_{\mathcal X}$ and $\mathbf{h}_s$ remain fixed during the forecasting horizon. Together, they summarize the pre-launch contextual information available before product-specific observations are incorporated.

To summarize evolving product-specific information, we use an autoregressive recurrent encoder. In general, this encoder can be implemented as either a Long Short-Term Memory (LSTM) network or a Gated Recurrent Unit (GRU). Since the GRU is used repeatedly in the remainder of the paper, we briefly note its role here. A GRU is a gated recurrent architecture that updates the hidden state through reset and update gates. This structure allows the model to preserve relevant past information while efficiently incorporating newly observed outcomes. In our implementation, the recurrent state encoder is instantiated as a GRU, although we retain the generic notation $\mathrm{RNN}_{\phi}$ for exposition.

Let $\mathbf{c}:=[\mathbf{h}_{\mathcal{X}};\mathbf{h}_s]$. We use the contextual embedding both to initialize the recurrent state and to provide persistent contextual input during the sequential update. Specifically, the initial hidden state is defined as $\mathbf{h}_0=\mathbf{W}_0\mathbf{c}+\mathbf{b}_0$, where $\mathbf{W}_0$ maps the contextual embedding to the GRU hidden dimension. At each time step $t$, the hidden state is updated by combining the current observation $\mathbf{x}_t$ with the fixed contextual embedding $\mathbf{c}$:
\begin{equation}\label{eq:def_f_theta}
\mathbf{h}_t
=
f_{\phi}(\mathbf{h}_{t-1},\mathbf{x}_t,\mathbf{c})
:=
\mathrm{RNN}_{\phi}\!\left([\mathbf{x}_t;\mathbf{c}],\mathbf{h}_{t-1}\right).
\end{equation}
This construction ensures that static and cross-product contextual information remain available throughout the sequential update process, rather than relying only on a one-time initialization. It is precisely this evolving latent state that enables the model to move from cross-product transfer in the pre-launch phase to product-specific adaptation in the early post-launch phase.

\subsection{Conditional Diffusion and Joint Training}\label{sec42}

We now describe the conditional diffusion model and its joint training with the state-transition encoder. Throughout this subsection, $q(\cdot)$ denotes the fixed forward noising process, while $p_{\theta}(\cdot)$ denotes the learned reverse denoising model. For each forecasting step $t\ge t_0$, let $\mathbf{x}_t^{0}\in\mathbb{R}^{D}$ denote the clean outcome vector at time $t$, and let $\mathbf{x}_t^{0:N}$ denote the associated diffusion chain over noise levels $n=0,1,\ldots,N$. The reverse denoising process is conditioned on the latent state $\mathbf{h}_{t-1}$ constructed in Section~\ref{sec31}.

The diffusion component models the uncertainty of the next-period outcome, while the latent state ensures that generation remains consistent with the information available at the current forecast origin. This conditioning is the main feature that makes the diffusion model specific to the cold-start forecasting problem. The generated outcome at time $t$ depends not only on the noisy diffusion state itself, but also on the static descriptors, the reference-product information, and, when available, the observed prefix summarized by $\mathbf{h}_{t-1}$.

The forward diffusion process is defined as a fixed Gaussian Markov chain:
\[
q(\mathbf{x}_t^{n}\mid \mathbf{x}_t^{n-1})
=
\mathcal{N}\!\left(
\mathbf{x}_t^{n};
\sqrt{1-\beta_n}\,\mathbf{x}_t^{n-1},
\beta_n\mathbf{I}
\right),
\]
with variance schedule $\{\beta_n\}_{n=1}^{N}$ and cumulative product 
$\bar{\alpha}_n=\prod_{i=1}^{n}(1-\beta_i)$. It follows that the forward marginal at any diffusion step $n$ has the closed form
\begin{equation}\label{eq31}
q(\mathbf{x}_t^{n}\mid \mathbf{x}_t^{0})
=
\mathcal{N}\!\left(
\mathbf{x}_t^{n};
\sqrt{\bar{\alpha}_n}\,\mathbf{x}_t^{0},
(1-\bar{\alpha}_n)\mathbf{I}
\right),
\qquad
\mathbf{x}_t^{n}
=
\sqrt{\bar{\alpha}_n}\,\mathbf{x}_t^{0}
+
\sqrt{1-\bar{\alpha}_n}\,\boldsymbol{\epsilon},
\end{equation}
where $\boldsymbol{\epsilon}\sim\mathcal{N}(\mathbf{0},\mathbf{I})$.

Under the standard DDPM construction, the posterior transition
$q(\mathbf{x}_t^{n-1}\mid \mathbf{x}_t^{n},\mathbf{x}_t^{0})$
is Gaussian with tractable mean and variance:
\[
q(\mathbf{x}_t^{n-1}\mid \mathbf{x}_t^{n},\mathbf{x}_t^{0})
=
\mathcal{N}\!\left(
\mathbf{x}_t^{n-1};
\widetilde{\mu}_n(\mathbf{x}_t^{n},\mathbf{x}_t^{0}),
\sigma_n^2\mathbf{I}
\right),
\]
where
\begin{equation}\label{eq4}
\widetilde{\mu}_n(\mathbf{x}_t^{n},\mathbf{x}_t^{0})
:=
\frac{\sqrt{\bar{\alpha}_{n-1}}\beta_n}{1-\bar{\alpha}_n}\mathbf{x}_t^{0}
+
\frac{\sqrt{\alpha_n}(1-\bar{\alpha}_{n-1})}{1-\bar{\alpha}_n}\mathbf{x}_t^{n},
\qquad
\sigma_n^2
:=
\frac{1-\bar{\alpha}_{n-1}}{1-\bar{\alpha}_n}\beta_n.
\end{equation}

In practice, neither the clean target $\mathbf{x}_t^{0}$ nor the injected noise $\boldsymbol{\epsilon}$ is observed along the reverse chain. We therefore parameterize the reverse denoising transition by a neural noise-prediction network $s_{\theta}$, also referred to as a score network in the diffusion literature. In contrast to unconditional diffusion models, our reverse process is conditioned on the autoregressive latent state $\mathbf{h}_{t-1}$:
\[
p_{\theta}(\mathbf{x}_t^{n-1}\mid \mathbf{x}_t^{n},\mathbf{h}_{t-1})
=
\mathcal{N}\!\left(
\mathbf{x}_t^{n-1};
\mu_{\theta}(\mathbf{x}_t^{n},\mathbf{h}_{t-1},n),
\sigma_n^2\mathbf{I}
\right).
\]
Using the standard DDPM reparameterization, the reverse mean can be written as
\begin{equation}\label{eq6}
\mu_{\theta}(\mathbf{x}_t^{n},\mathbf{h}_{t-1},n)
=
\frac{1}{\sqrt{\alpha_n}}
\left(
\mathbf{x}_t^{n}
-
\frac{\beta_n}{\sqrt{1-\bar{\alpha}_n}}
s_{\theta}(\mathbf{x}_t^{n},\mathbf{h}_{t-1},n)
\right).
\end{equation}

We implement the score network $s_{\theta}$ as a one-dimensional dilated causal convolutional neural network with $L$ residual blocks. At time $t$, the network takes as input a noisy local window
\[
\mathbf{X}_t^{n}
=
[\mathbf{x}_{t-W+1}^{n},\ldots,\mathbf{x}_t^{n}]
\in\mathbb{R}^{W\times D},
\]
together with the latent state $\mathbf{h}_{t-1}$ and a diffusion-step embedding $\mathbf{e}(n)$. The network outputs a noise estimate $\widehat{\boldsymbol{\epsilon}}_t\in\mathbb{R}^{D}$, which is used in \eqref{eq6} to parameterize the reverse transition. The causal convolutional score network preserves temporal ordering within the local window, while the latent state injects longer-range contextual and cross-product information. In this way, the model separates local temporal denoising from global contextual adaptation. 

We train the score network $s_{\theta}$ and the state-transition encoder $f_{\phi}$ jointly. The target of interest at each time step is the conditional distribution $p_{\theta}(\mathbf{x}_t^{0}\mid \mathbf{h}_{t-1})$, where $\mathbf{h}_{t-1}$ summarizes the information available before predicting $\mathbf{x}_t^{0}$. The corresponding reverse diffusion model is
\[
p_{\theta}(\mathbf{x}_t^{0:N}\mid \mathbf{h}_{t-1})
=
p(\mathbf{x}_t^{N})\prod_{n=1}^{N}
p_{\theta}(\mathbf{x}_t^{n-1}\mid \mathbf{x}_t^{n},\mathbf{h}_{t-1}),
\qquad
p(\mathbf{x}_t^{N})=\mathcal{N}(\mathbf{0},\mathbf{I}).
\]
Marginalizing the intermediate noisy states gives
\[
p_{\theta}(\mathbf{x}_t^{0}\mid \mathbf{h}_{t-1})
=
\int
p_{\theta}(\mathbf{x}_t^{0:N}\mid \mathbf{h}_{t-1})\,d\mathbf{x}_t^{1:N}.
\]

In the pre-launch phase, forecasting starts from $t_0=1$, since no product-specific observations are available. In the early post-launch phase, $t_0$ denotes the first unobserved step after the observed prefix. This shared formulation allows the model to handle both cold-start phases within a single training framework.

A natural estimation principle is conditional maximum likelihood: $ \max_{\theta,\phi}
\sum_{t=t_0}^{T}\log p_{\theta}(\mathbf{x}_t^{0}\mid \mathbf{h}_{t-1})$. However, this objective is intractable because the likelihood requires integration over the latent diffusion states $\mathbf{x}_t^{1:N}$. To obtain a tractable objective, we use the fixed forward process
$
q(\mathbf{x}_t^{1:N}\mid \mathbf{x}_t^{0})
=
\prod_{n=1}^{N} q(\mathbf{x}_t^{n}\mid \mathbf{x}_t^{n-1})$, and train the score network to predict the injected Gaussian noise. Following standard DDPM and TimeGrad practice, we optimize the denoising objective
\begin{equation}\label{eq:joint_train_objective}
\min_{\theta,\phi}
\ \mathcal{L}_{\mathrm{train}}(\theta,\phi)
:=
\mathbb{E}_{t,\mathbf{x}_t^{0},\,n,\,\boldsymbol{\epsilon}}
\Big[
\|s_{\theta}(\mathbf{x}_t^{n},\mathbf{h}_{t-1},n)-\boldsymbol{\epsilon}\|^2
\Big],
\end{equation}
Where $\mathbf{x}_t^{n}
=
\sqrt{\bar{\alpha}_n}\,\mathbf{x}_t^{0}
+
\sqrt{1-\bar{\alpha}_n}\,\boldsymbol{\epsilon}$, $\boldsymbol{\epsilon}\sim\mathcal{N}(\mathbf{0},\mathbf{I})$.

The hidden state $\mathbf{h}_{t-1}$ is computed recursively through the state-transition encoder defined in \eqref{eq:def_f_theta}. Therefore, gradients from the denoising loss backpropagate not only through the score network $s_{\theta}$, but also through the recurrent encoder $f_{\phi}$. This joint training is important because the latent state should be optimized for the downstream denoising task itself, rather than learned independently of it. Algorithm~\ref{alg:training_diffusion} summarizes the training procedure under full trajectory access. During training, we sample a forecast origin $t_0$, construct the corresponding hidden state from the observed prefix up to $t_0-1$, sample a diffusion level $n$, and optimize the noise-prediction loss at that forecast origin.

At inference time, the trained model is applied autoregressively across time. Given $\mathbf{h}_{t-1}$, we generate $\mathbf{x}_t^{0}$ by running the conditional reverse diffusion chain from $\mathbf{x}_t^{N}\sim\mathcal{N}(\mathbf{0},\mathbf{I})$, then update the latent state to $\mathbf{h}_t=f_{\phi}(\mathbf{h}_{t-1},\mathbf{x}_t^{0},\mathbf{c})$, and repeat to obtain $\mathbf{x}_{t+1}^{0},\mathbf{x}_{t+2}^{0},\ldots$. Importantly, both the score network and the state-transition encoder share parameters across all time steps, which promotes temporal consistency and knowledge transfer across different stages of the product life-cycle.

\begin{algorithm}[htb]
\caption{Training Procedure with Full Trajectory Access}\label{alg:training_diffusion}
\begin{algorithmic}
\STATE \textbf{Input:} Full product observation trajectory $\{\mathbf{x}_1^0,\ldots,\mathbf{x}_T^0\}$, aggregated reference embedding $\mathbf{h}_{\mathcal X}$, static embedding $\mathbf{h}_s$
\STATE \textbf{while} the network has not converged \textbf{do}
\STATE \hspace{0.5cm} Sample time step $t_0 \sim \mathcal{U}(1,\dots,T)$
\STATE \hspace{0.5cm} Set $\mathbf{c}=[\mathbf{h}_{\mathcal X};\mathbf{h}_s]$ and initialize $\mathbf{h}_0=\mathbf{W}_0\mathbf{c}+\mathbf{b}_0$
\STATE \hspace{0.5cm} \textbf{if} $t_0>1$ \textbf{then}
\STATE \hspace{1cm} \textbf{for} $t=1,\dots,t_0-1$ \textbf{do}
\STATE \hspace{1.5cm} $\mathbf{h}_t=f_{\phi}(\mathbf{h}_{t-1},\mathbf{x}_t^0,\mathbf{c})$
\STATE \hspace{1cm} \textbf{end for}
\STATE \hspace{0.5cm} \textbf{end if}
\STATE \hspace{0.5cm} Sample diffusion step $n\sim \mathcal{U}(1,\dots,N)$ and Gaussian noise $\boldsymbol{\epsilon}\sim\mathcal{N}(\mathbf{0},\mathbf{I})$
\STATE \hspace{0.5cm} Construct $\mathbf{x}_{t_0}^{n}=\sqrt{\bar{\alpha}_n}\mathbf{x}_{t_0}^0+\sqrt{1-\bar{\alpha}_n}\boldsymbol{\epsilon}$
\STATE \hspace{0.5cm} Take a gradient step on $\left\|
s_{\theta}(\mathbf{x}_{t_0}^{n},\mathbf{h}_{t_0-1},n)-\boldsymbol{\epsilon}
\right\|^2$ with respect to $(\theta,\phi)$
\STATE \textbf{end while}
\end{algorithmic}
\end{algorithm}
\subsection{Theoretical Analysis}
\label{sec:theory_main_bilingual}

In this section, we present error analysis on the proposed generative model.
We first formalize the \textit{best-in-class} model that serves as the benchmark,
and then introduce appropriate error metrics for comparing the trained model against this benchmark.
Recall that over the forecasting horizon $t = t_0,\dots,T$, the model generates predictions recursively as
\begin{equation}
\label{eq:ar_gen_rule}
\widehat{\mathbf{x}}_t \sim p_\theta(\cdot \mid \widehat{\mathbf{h}}_{t-1}),\qquad
\widehat{\mathbf{h}}_t = f_\phi(\widehat{\mathbf{h}}_{t-1}, \widehat{\mathbf{x}}_t, \mathbf{c}),\qquad
\mathbf{c} := [\mathbf{h}_{\mathcal{X}}; \mathbf{h}_{\mathbf{s}}].
\end{equation}
Here $p_\theta$ denotes the trained conditional generative model,
and $f_\phi$ denotes the recurrent transition network. Throughout this subsection, notations with hats (e.g., $\widehat{\mathbf{x}}_t$, $\widehat{\mathbf{h}}_t$) denote quantities generated by the trained diffusion model.

Let $(\theta^\star,\phi^\star)\in\arg\min_{\theta,\phi}\mathcal L_{\mathrm{train}}(\theta,\phi)$ be the model parameter that minimizes the training loss, and let $f^\star:=f_{\phi^\star}$, $s^\star:=s_{\theta^\star}$, $p^\star(\cdot\mid\mathbf h):=p_{\theta^\star}(\cdot\mid\mathbf h)$ denote the corresponding \textit{best-in-class} models. Then the \textit{best-in-class} latent state evolves as $\mathbf h_t^\star = f^\star(\mathbf h_{t-1}^\star, \mathbf x_t, \mathbf c)$. 
We measure distributional discrepancies via the 1-Wasserstein distance $W_1(P,Q):=\inf_{\pi\in\Pi(P,Q)}\mathbb{E}_{(\mathbf X,\mathbf Y)\sim\pi}[\|\mathbf X-\mathbf Y\|]$.
The key error metrics are the expected one-step marginal forecast error,
\begin{equation}
\label{eq:delta_E_def}
\Delta_t
:=\mathbb{E}\Big[
W_1\!\left(
p_\theta(\cdot\mid \widehat{\mathbf h}_{t-1}),
p^\star(\cdot\mid \mathbf h_{t-1}^\star)
\right)
\Big],
\end{equation}
the latent-state error $E_t:=\mathbb{E}\big[\|\widehat{\mathbf h}_t-\mathbf h_t^\star\|\big]$,
and the worst-case marginal error over the horizon, defined as $\overline{\Delta}:=\sup_{t_0\le t\le T}\Delta_t$.
Technical conditions ensuring $\Delta_t$ is well-defined are given in Appendix~\ref{app:reachable_set}.
Before proceeding, some structural assumptions are required.

\begin{assumption}
    \label{assump:combined_main}
    There exist constants $L_P>0$, $\rho\in[0,1)$, and $L_x\ge 0$ such that:
    \begin{enumerate}
    \item[(1)] For all $t$ and any $\mathbf{h},\mathbf{h}'\in\mathcal{H}$,
    $
    W_1\!\left(p_t^\star(\cdot\mid \mathbf{h}),p_t^\star(\cdot\mid \mathbf{h}')\right)\le L_P\|\mathbf{h}-\mathbf{h}'\|;$
    \item[(2)] For any $\mathbf{h},\mathbf{h}'\in\mathcal{H}$ and any $\mathbf{x},\mathbf{x}'$ in the relevant input set,
    $\|f_\phi(\mathbf{h},\mathbf{x},\mathbf{c})-f_\phi(\mathbf{h}',\mathbf{x}',\mathbf{c})\|
    \le \rho\|\mathbf{h}-\mathbf{h}'\| + L_x\|\mathbf{x}-\mathbf{x}'\|.$
    \end{enumerate}
    \end{assumption}
    
Assumption~\ref{assump:combined_main} collects three stability-related constants. The constant $L_P$ bounds the sensitivity of the best-in-class one-step conditional distribution to latent-state perturbations; $\rho$ is the contraction factor of the recurrent update in the latent state; and $L_x$ controls the Lipschitz dependence on the input signal. In our theoretical analysis (Theorem~\ref{thm:multi_step_consistency_main}), we require that $\rho<1$ and that $L_P$ and $L_x$ are bounded by explicit thresholds, which together ensure that multi-step prediction errors do not explode.
These conditions are mild and can be satisfied by design. In practice, the constants $L_P$, $\rho$, and $L_x$ can be controlled via architectural constraints and regularization techniques on the neural network, including spectral normalization, orthogonal RNN, and Lipschitz-bounded activation functions \citep{XuSivaranjani2024ECLipsE,DelattreJacquemont2023GramIteration,zadorozhnyy2024orthogonal}. In our implementation, we follow the enforcement recipe in Appendix~\ref{app:gru_stability}, which applies spectral normalization during training, monitors the empirical contraction factor on rollout trajectories, and re-initializes when it exceeds a conservative stability threshold.
These implementation choices are intended to promote the stability conditions in Assumption \ref{assump:combined_main} and make them more plausible in practice.

We now define the two error terms that appear in the final bound. The \emph{generative error} $\varepsilon_{\mathrm{gen}}$ measures the uniform discrepancy between the learned conditional generator $p_\theta(\cdot\mid\mathbf{h})$ and the best-in-class $p^\star(\cdot\mid\mathbf{h})$ over the reachable set $\mathcal{H}$:
\begin{equation}
\label{eq:eps_gen_def_main}
\varepsilon_{\mathrm{gen}}
\;:=\;
\sup_{\mathbf{h}\in\mathcal{H}}\;
W_1\!\Big(p_\theta(\cdot\mid\mathbf{h}),\;p^\star(\cdot\mid\mathbf{h})\Big),
\end{equation}
and the \emph{transition estimation error} $\varepsilon_f$ captures how closely the learned recurrent transition $f_\phi$ approximates the best-in-class $f^\star$ at the same inputs:
\begin{equation}
\label{eq:eps_f_def_main}
\varepsilon_f
\;:=\;
\sup_{t}\;
\mathbb{E}\Big[
\big\|
f_\phi(\mathbf{h}_{t-1}^\star,\mathbf{x}_t,\mathbf{c})
-
f^\star(\mathbf{h}_{t-1}^\star,\mathbf{x}_t,\mathbf{c})
\big\|
\Big].
\end{equation}
As the model scales and data accumulate, both terms vanish under standard regularity conditions:
$\varepsilon_{\mathrm{gen}}\to 0$ as the score network's capacity grows and the diffusion time discretization is refined \citep{gao2025wasserstein,wang2024wasserstein},
and $\varepsilon_f\to 0$ as the RNN transition model is trained on larger samples \citep{chen2019generalization,jiao2024approximation}. The explicit functional forms of both error bounds are omitted for brevity and can be found in the references above.

\begin{theorem}[Forecasting Error Bound]
    \label{thm:multi_step_consistency_main}
    Define $\kappa := \frac{L_P L_x}{1-\rho}$. 
    If Assumption~\ref{assump:combined_main} holds and $\kappa < 1$, then
    \begin{equation}
    \label{eq:main_uniform_bound}
    \overline{\Delta}
    :=
    \sup_{t_0\le t\le T}\Delta_t
    \;\le\;
    \frac{1}{1-\kappa}
    \left(
    \varepsilon_{\mathrm{gen}}
    +\frac{L_P}{1-\rho}\varepsilon_f
    +L_P E_{t_0-1}
    \right),
    \end{equation}
    where $E_{t_0-1} := \mathbb{E}\|\widehat{\mathbf{h}}_{t_0-1}-\mathbf{h}_{t_0-1}^\star\|$ denotes the initial latent-state error at time $t_0-1$, i.e., at the beginning of the forecasting horizon.
\end{theorem}

The stability margin $\kappa = \frac{L_P L_x}{1-\rho}$ governs whether multi-step forecasting error remains controlled. When $\kappa \ge 1$, the right-hand side of \eqref{eq:main_uniform_bound} becomes unbounded, and no finite guarantee can be obtained; thus $\kappa<1$ is a necessary condition for any meaningful guarantee. When $\kappa<1$, the bound ensures that the horizon-uniform error $\overline{\Delta}$ remains finite, indicating that the autoregressive feedback loop is contractive rather than amplifying.
The three constants ($L_P,L_x,\rho$) in $\kappa$ admit a direct interpretation: $L_P$ quantifies how sensitive the target conditional law is to latent-state perturbations; $L_x$ captures how input forecast mismatch propagates into latent drift through the recurrent transition; and $1-\rho$ measures the contraction strength of the latent dynamics, i.e., how aggressively the GRU discounts past latent mismatch. In practice, $L_P$ is influenced by architectural choices (for example, a smaller latent dimension may reduce sensitivity), while $\rho$ and $L_x$ are jointly controlled through spectral regularization of the GRU weight matrices.
Specifically, Appendix~\ref{app:gru_stability} makes the condition $\kappa<1$ checkable for the GRU update by bounding $(\rho,L_x)$ via Jacobian spectral-norm conditions that depend only on weight-matrix norms and observable gate ranges. The resulting sufficient condition is fully explicit and can be enforced during training via spectral normalization or spectral clipping on the recurrent weight matrices. Our analysis therefore provides a principled guideline for designing stable neural network architectures. Appendix~\ref{app:sensitivity_kappa} further provides a numerical sensitivity analysis, suggesting a practical safe range $\kappa\in[0.5,0.8]$ that balances stability and expressive capacity.

   %Theorem~\ref{thm:multi_step_consistency_main} further shows that multi-step forecasting error vanishes as the model improves, i.e., $\overline{\Delta}\to 0$, provided that (i) the generative approximation error $\varepsilon_{\mathrm{gen}}$ and the transition estimation error $\varepsilon_f$ both vanish, and (ii) the initial latent-state mismatch at the start of multi-step forecasting, $E_{t_0-1}$, also vanishes. Here $E_{t_0-1}$ measures the discrepancy between the learned and best-in-class latent states at time $t_0-1$. This condition is necessary because multi-step forecasting inherits any mismatch at its starting point. 

   Theorem~\ref{thm:multi_step_consistency_main} further shows that multi-step forecasting error is ultimately governed by the initial latent-state mismatch at the start of forecasting. In particular, $\overline{\Delta}$ converges to a level determined by $E_{t_0-1}$ as the model improves. When the generative error $\varepsilon_{\mathrm{gen}}$ and the transition error $\varepsilon_f$ vanish, the remaining forecasting error is driven entirely by $E_{t_0-1}$. Here $E_{t_0-1}$ measures the discrepancy between the learned and best-in-class latent states at time $t_0-1$, and reflects the fact that multi-step forecasting inherits any mismatch at its starting point.

\section{Empirical Studies}\label{sec5}

We evaluate CDLF in two complementary cold-start settings under a shared launch-stage protocol. The two studies serve different but complementary purposes. Study~1 considers Intel microprocessor SKU life-cycles, a classical benchmark setting with rich descriptors, long trajectories, and strong prior baselines from the literature. It therefore provides a disciplined setting for assessing whether CDLF improves both point and probabilistic forecasting performance relative to established alternatives. Study~2 is closer to the central motivation of this paper. It examines the platform-mediated adoption of open LLM repositories on the Hugging Face Hub, where early adoption is highly heterogeneous, platform exposure is episodic, and bursts in attention can quickly reshape the trajectory. This second setting is especially relevant in the current AI environment, in which many models and related resources are released continuously, yet their adoption paths differ sharply in scale, timing, and persistence. Across both studies, we compare CDLF against standard forecasting baselines, evaluate both point and distributional accuracy, and conduct ablations to isolate the roles of $\mathbf{s}$ and $\mathcal{X}$.

\subsection{Study 1: Intel Microprocessor Life-Cycle Forecasting}

We first evaluate CDLF on Intel microprocessor SKU life-cycles. Our goal is to examine whether conditioning on static descriptors $\mathbf{s}$ and reference trajectories $\mathcal{X}$ improves early-stage forecasting accuracy and probabilistic calibration in a canonical product life-cycle setting.

\subsubsection{Setup: Data, Protocol, Benchmarks, and Metrics}

We use the Intel microprocessor sales dataset of \cite{Manary2022}, which is also adopted by \cite{Guo2025}. The dataset contains weekly sales records for 86 Intel CPU SKUs, with up to 187 weeks per SKU. The median (mean) life-cycle length is 109 (86) weeks, and the median (mean) weekly orders are 22{,}850 (43{,}370). Each SKU is associated with static product descriptors, including generation, core count, clock speed, and average selling price. These descriptors provide contextual information when the target SKU has little or no product-specific history.

To improve comparability across heterogeneous scales, we normalize each SKU's weekly sales by its own maximum observed weekly sales. We further align product launches to a common week zero so that all trajectories are represented on a standardized life-cycle timeline. Cold-start evaluation controls how much post-launch history is revealed at test time, distinguishing between two phases: (i) \emph{pre-launch} forecasting, in which no target-SKU observations are revealed, and (ii) \emph{early post-launch} forecasting, in which only a short observed prefix is available. Within CDLF, each forecasting instance is represented by three components: (i) the static descriptor vector $\mathbf{s}$ for the target SKU; (ii) a reference set of historical life-cycle trajectories $\mathcal{X}=\{(\mathbf{x}_{1:T}^{(1)},\mathbf{s}^{(1)}),\dots,(\mathbf{x}_{1:T}^{(K)},\mathbf{s}^{(K)})\}$
selected from previously launched products; and (iii) an observed prefix $\mathbf{x}_{1:t_0-1}$ of the target trajectory, which is revealed only in the early post-launch setting and provides progressively richer product-specific evidence as $t_0$ increases.

We compare CDLF against representative baselines spanning classical diffusion modeling, analog-ensemble forecasting, Bayesian updating, and nonparametric machine learning: (i) Bass diffusion \citep{Bass1969}; (ii) the TiGo-ETS analog ensemble of \cite{Guo2025}; (iii) the Bayesian nonparametric (Bayesian NP) model of \cite{Dew2018}; and (iv) Quantile Regression Forests (QRF; \citealt{meinshausen2006a}). All methods use identical train/test splits and the same rolling forecasting procedure. We report short-horizon performance (1--8 weeks ahead) and medium-horizon performance (9--16 weeks ahead). For TiGo-ETS, we adopt the published smoothing parameters and clustering settings in \cite{Guo2025}. QRF is trained with 1{,}000 trees and optimized using pinball loss. For Bayesian NP, because customer-level data are unavailable, we assume that customers begin considering the product at launch and adopt at most once; market potential is approximated by the median demand among comparable historical life-cycles.

Following \cite{Guo2025}, we evaluate both point accuracy and probabilistic forecast quality. Point forecasts are assessed by mean absolute error (MAE), computed between the realized weekly sales $x_t^0$ and a central summary of the predictive distribution, such as the median. For probabilistic forecasts, we use pinball loss \citep{Koenker1978} for quantile predictions. Given realized $x_t^0$ and predicted $u$-quantile $\hat{x}_{t,u}$,
\[
L_u(\hat{x}_{t,u}, x_t^0)=
\begin{cases}
u (x_t^0-\hat{x}_{t,u}), & x_t^0 \ge \hat{x}_{t,u},\\
(1-u)(\hat{x}_{t,u}-x_t^0), & x_t^0 < \hat{x}_{t,u}.
\end{cases}
\]
For CDLF, $\hat{x}_{t,u}$ is obtained empirically from samples drawn from the conditional predictive distribution:
\[
x_{1,t},\dots,x_{M,t}\sim p_\theta(x_t \mid \mathbf{h}_{t-1}),
\qquad
\hat{x}_{t,u}=\mathrm{Quantile}_u(\{x_{m,t}\}_{m=1}^{M}).
\]
We summarize distributional accuracy using the continuous ranked probability score (CRPS; \citealt{Matheson1976}), approximated by the mean pinball loss over a dense grid of quantiles \citep{Gneiting2011}:
\[
\mathrm{CRPS}\big(p_\theta(x_t \mid \mathbf{h}_{t-1}), x_t^0\big)
\approx
\frac{1}{99}\sum_{j=1}^{99} L_{u_j}(\hat{x}_{t,u_j}, x_t^0),
\qquad
u_j=\frac{j}{100}.
\]
Finally, MCRPS denotes the mean CRPS averaged over forecast times and SKUs.

\subsubsection{Results and Discussion}

Table~\ref{tab:study1} reports rolling out-of-sample results for short (1--8 weeks) and medium (9--16 weeks) horizons using MAE and MCRPS ($\times 10^3$). CDLF achieves the best performance on all four metrics. Relative to the strongest baseline, TiGo-ETS, CDLF reduces MAE from 17.16 to 14.52 for 1--8 weeks ahead and from 23.61 to 19.96 for 9--16 weeks ahead. It also improves probabilistic accuracy, reducing MCRPS from 13.67 to 10.16 for 1--8 weeks ahead and from 19.03 to 15.12 for 9--16 weeks ahead. These gains indicate that CDLF improves both central forecast accuracy and uncertainty quantification. Among the benchmarks, TiGo-ETS is the strongest classical baseline and QRF is the most competitive nonparametric machine-learning baseline. However, neither matches CDLF in either point accuracy or distributional quality, and the advantage of CDLF persists as the forecast horizon extends. This pattern is consistent with the intended role of CDLF in cold-start settings: static descriptors and analog life-cycle trajectories provide a stronger prior in the absence of long within-SKU histories, while the conditional generative formulation helps preserve predictive flexibility as uncertainty accumulates over time.

\begin{table}[htp]
\centering
\small
\caption{\label{tab:study1}Study~1 (Intel microprocessors): performance comparison across models under the rolling forecast setup for 1--8 and 9--16 weeks ahead. Metrics include MAE($\times10^3$) and MCRPS($\times10^3$).}
\begin{tabular}{cccccc}
\hline
\multirow{2}*{Forecasting Approach} & \multirow{2}*{Model} & \multicolumn{2}{c}{1--8 weeks ahead} & \multicolumn{2}{c}{9--16 weeks ahead}\\
\cline{3-6}
& & MAE & MCRPS & MAE & MCRPS\\
\hline
\multirow{2}*{Bayesian Ensembles of Single Model Type}
&Bass & 30.82 & 27.05 & 35.12 & 31.96\\
&TiGo-ETS & 17.16 & 13.67 & 23.61 & 19.03\\
\hline
Bayesian Updating Approaches
&Bayesian NP & 26.78 & 21.55 & 33.81 & 29.42\\
\hline
\multirow{2}*{Machine Learning Approaches}
&QRF & 21.02 & 17.51 & 26.98 & 23.84\\
&CDLF & \textbf{14.52} & \textbf{10.16} & \textbf{19.96} & \textbf{15.12}\\
\hline
\end{tabular}
\end{table}

To further examine distributional performance across the demand range, Figure~\ref{fig:study1} visualizes the mean pinball loss across 99 quantiles. CDLF dominates over most quantiles, with TiGo-ETS and QRF as the closest competitors. Accordingly, the subsequent visual comparison focuses on these three methods for clarity. The quantile-wise results reinforce the main message of Table~\ref{tab:study1}: the gain of CDLF is not confined to a single point summary, but extends across much of the predictive distribution.

\begin{figure}[htbp]
\centering
\subfigure[1--8 Weeks Ahead]{
\includegraphics[width=2.5in]{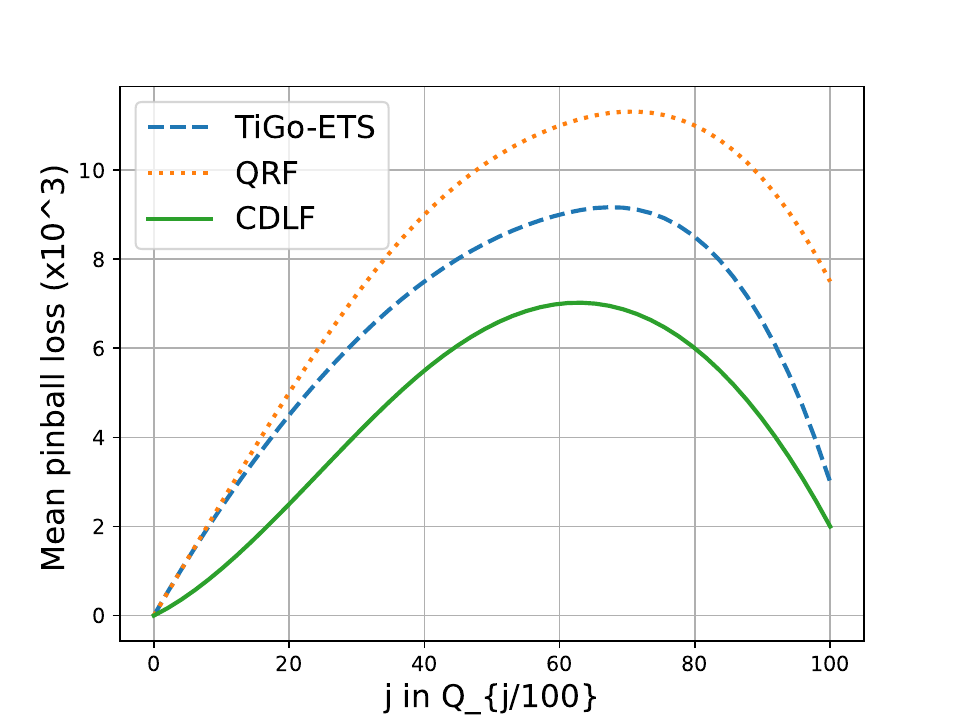}
}
\subfigure[9--16 Weeks Ahead]{
\includegraphics[width=2.5in]{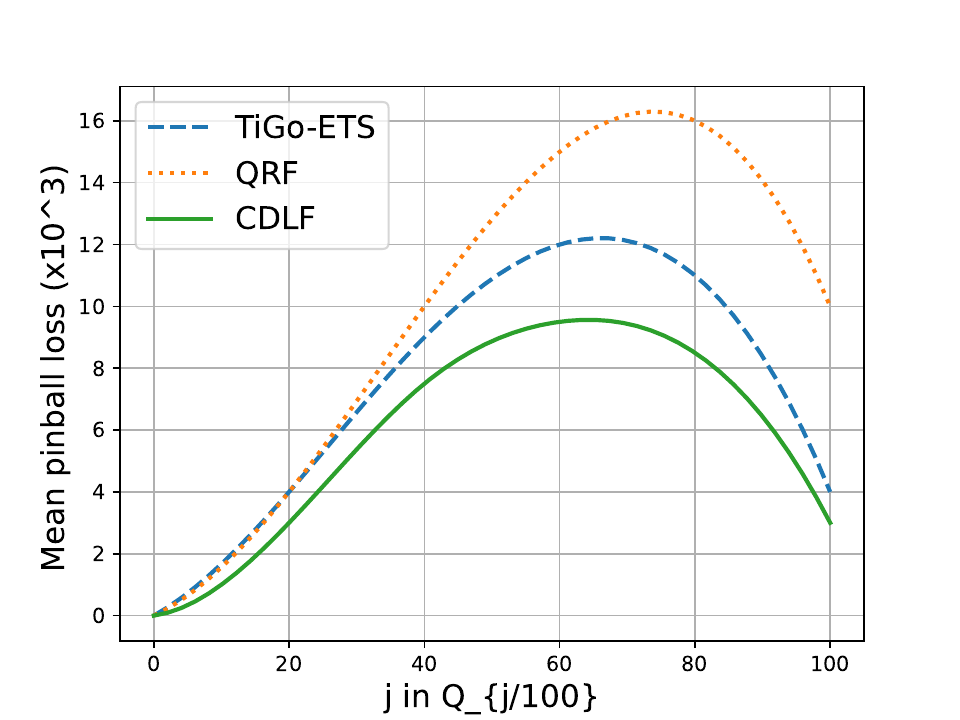}
}
\caption{Study~1: quantile-wise calibration measured by mean pinball loss across 99 quantiles.}
\label{fig:study1}
\end{figure}

% =========================================================

\subsection{Study 2: Cold-Start Adoption Forecasting for LLM Repositories on the Hugging Face Hub}
\label{sec:study2_hf_llm}

In Study~2, we treat  the adoption dynamics of open LLM repositories on the Hugging Face Hub as platform-mediated digital products and use public engagement signals as demand proxies. The goal is to evaluate whether CDLF can forecast multi-step adoption trajectories under both \emph{pre-launch} and \emph{early post-launch} information phases in a setting characterized by strong heterogeneity and bursty exposure dynamics.

\subsubsection{Data, CDLF Mapping, and Evaluation Protocol}

We construct a daily adoption panel from the public dataset \texttt{severo/trending-repos}, which records the top trending repositories on the Hugging Face Hub. For the \texttt{models} configuration, each record includes \texttt{date}, \texttt{id}, \texttt{rank}, \texttt{recent\_likes}, \texttt{likes} (cumulative), and \texttt{month\_downloads}. To focus on LLMs, we restrict attention to repositories with \texttt{pipeline\_tag = text-generation} and curate a set of 12 repositories, yielding 863 repo-day observations from 2023-07-29 to 2025-06-01. Since platform visibility is episodic rather than continuous, we segment each repository series into contiguous trending episodes. Across the 12 repositories, we identify 44 episodes and retain the 24 episodes with length at least 10 days in order to standardize the prefix and forecast horizon.

Within each retained episode, we align the episode start to a common origin and define the daily adoption increment as
\[
\Delta_t := \text{likes}_t - \text{likes}_{t-1},
\qquad
x_t := \log(1 + \Delta_t),
\]
where $\text{likes}_t$ is the cumulative like count on day $t$. For the first day of an episode, when $\text{likes}_{t-1}$ is unavailable, we set $\Delta_1=0$ and hence $x_1=0$. The log transform stabilizes heavy tails while preserving comparability across repositories. The resulting target trajectory is $\mathbf{x}_{1:T}$. We use a compact set of static descriptors derived deterministically from public metadata:
\[
\mathbf{s}:=\big(\mathrm{Org},\ \mathrm{ScaleBucket},\ \mathrm{Gated}\big),
\]
where $\mathrm{Org}$ is the repository namespace, $\mathrm{Gated}\in\{0,1\}$ indicates whether access requires acceptance of usage terms, and $\mathrm{ScaleBucket}$ is parsed from the repository identifier when available (e.g., 7B, 8B, 70B, or mixture-of-experts patterns such as 8x7B and 8x22B). These descriptors are intentionally compact. They are not meant to capture all determinants of adoption, but rather to provide a minimal and comparable representation that can support cold-start transfer across repositories.

Following the CDLF mapping, we construct the analogical reference set $\mathcal{X}$ from other repositories in Table~\ref{tab:hf_llm_products} using similarity over $\mathbf{s}$, implemented through top-$K$ nearest neighbors in the descriptor space. In the early post-launch setting, the model additionally conditions on an observed prefix $\mathbf{x}_{1:t_0-1}$. If desired, public platform signals such as daily rank, recent likes, and monthly downloads can also be appended to $\mathbf{x}_t$ as additional channels, although our main specification focuses on the core adoption trajectory.

For evaluation, we adopt the same launch-stage protocol as in Study~1. In the \emph{pre-launch} setting, we forecast the full horizon $\mathbf{x}_{1:T}$ using only $\mathbf{s}$ and $\mathcal{X}$. In the \emph{early post-launch} setting, we reveal a short observed prefix $\mathbf{x}_{1:t_0-1}$ and forecast the remaining horizon. We set $T=10$ and $t_0=6$ for all instances, corresponding to a 5-day observed prefix and a 5-day forecast horizon. Performance is evaluated using MAE and MCRPS, with CRPS approximated by the mean pinball loss over 99 quantiles. Benchmarks include Bass, TiGo-ETS, Bayesian NP, and QRF, using the same information for each method.

\begin{table}[htp]
\centering
\small
\caption{Study~2 (Hugging Face LLM adoption): selected LLM repositories in the public trending panel.}
\label{tab:hf_llm_products}
\begin{tabular}{p{8.0cm}p{2.2cm}p{1.7cm}p{1.3cm}}
\hline
\textbf{Repo id (Hugging Face)} & \textbf{Org} & \textbf{Scale} & \textbf{Gated} \\
\hline
\texttt{Qwen/QwQ-32B-Preview} & Qwen & 32B & No \\
\texttt{deepseek-ai/DeepSeek-R1} & deepseek-ai & -- & No \\
\texttt{google/gemma-7b-it} & google & 7B & Yes \\
\texttt{meta-llama/Llama-2-70b-chat-hf} & meta-llama & 70B & Yes \\
\texttt{meta-llama/Llama-2-7b-chat-hf} & meta-llama & 7B & Yes \\
\texttt{meta-llama/Meta-Llama-3.1-70B-Instruct} & meta-llama & 70B & Yes \\
\texttt{meta-llama/Meta-Llama-3.1-8B-Instruct} & meta-llama & 8B & Yes \\
\texttt{microsoft/Phi-3-mini-4k-instruct} & microsoft & -- & No \\
\texttt{microsoft/phi-2} & microsoft & -- & No \\
\texttt{mistralai/Mistral-7B-Instruct-v0.2} & mistralai & 7B & No \\
\texttt{mistralai/Mixtral-8x22B-Instruct-v0.1} & mistralai & 8x22B & No \\
\texttt{mistralai/Mixtral-8x7B-Instruct-v0.1} & mistralai & 8x7B & No \\
\hline
\end{tabular}
\end{table}

\subsubsection{Results and Discussion}
\label{sec:study2_results_discussion}

This subsection reports the empirical results for cold-start LLM adoption forecasting on the Hugging Face trending panel under the protocol described above. All methods are evaluated on identical rolling windows with fixed horizon $T=10$ and cold-start point $t_0=6$, corresponding to a 5-day observed prefix and a 5-day forecast horizon. The evaluation set contains 12 repositories and 24 retained trending episodes. For diffusion-based models, we generate $M=100$ trajectories per test window for distributional evaluation. We organize the findings along five dimensions: overall forecasting performance, ablation on conditioning sources, feature-conditioned pre-launch generation, launch segmentation based on predictive distributions, and qualitative triangulation with public release rationales.

\textbf{Performance comparison and burst characterization.}
Table~\ref{tab:hf_model_comparison_event} reports point accuracy, distributional accuracy, and event-level metrics. CDLF outperforms all competing methods on MAE, RMSE, MCRPS, DTW, peak error, and AUC error. MAE, RMSE, and MCRPS are computed on the log-transformed scale / original scale, whereas peak and AUC errors are computed after inverse transformation to raw likes. Relative to the strongest baseline, TiGo-ETS, CDLF reduces MAE by 15.6\% ($0.45 \rightarrow 0.38$), improves MCRPS from 0.38 to 0.33, and lowers peak-intensity error by 21.4\%. These gains are substantively important because LLM adoption episodes are short, heterogeneous, and often burst-driven. A method that performs well only on average levels but fails to capture burst timing and magnitude is of limited value. The results therefore suggest that conditioning on both analog trajectories and product-specific prefixes is especially useful when early adoption is sparse but can change rapidly as exposure increases.

\begin{table}[htp]
\centering
\small
\caption{Study~2 (LLM adoption): performance comparison and event-level errors on the forecast horizon $[t_0,T]$ with $T=10$ and $t_0=6$ (5-day prefix, 5-day forecast). Metrics are averaged over rolling windows.}
\label{tab:hf_model_comparison_event}
\begin{tabular}{lcccccc}
\hline
Model & MAE $\downarrow$ & RMSE $\downarrow$ & MCRPS $\downarrow$ & DTW $\downarrow$ & Peak error $\downarrow$ & AUC error $\downarrow$ \\
\hline
\textbf{CDLF (ours)} & \textbf{0.38} & \textbf{0.62} & \textbf{0.33} & \textbf{1.35} & \textbf{55} & \textbf{180} \\
TiGo-ETS & 0.45 & 0.70 & 0.38 & 1.55 & 70 & 220 \\
QRF & 0.49 & 0.76 & 0.41 & 1.70 & 78 & 240 \\
Bayesian NP & 0.52 & 0.80 & 0.44 & 1.82 & 85 & 260 \\
Bass diffusion & 0.58 & 0.88 & 0.47 & 2.05 & 95 & 300 \\
\hline
\end{tabular}
\end{table}

\textbf{Case-based trajectory fidelity and uncertainty.}
To complement the aggregate metrics, Figure~\ref{fig:hf_case_combined} presents a representative episode-level comparison between the realized adoption increments $\Delta_t$ and the CDLF forecasts under two cold-start settings: pre-launch (no observed prefix) and early post-launch (a short observed prefix), together with P10--P90 predictive bands. Consistent with the improvements in MCRPS and DTW reported in Table~\ref{tab:hf_model_comparison_event}, the early post-launch forecast yields a tighter predictive band and tracks the turning points of the observed trajectory more closely. By contrast, the pre-launch forecast is more dispersed, as expected under greater information scarcity, but still provides a plausible baseline trajectory.

\begin{figure}[htbp]
  \centering
  \includegraphics[width=0.65\linewidth]{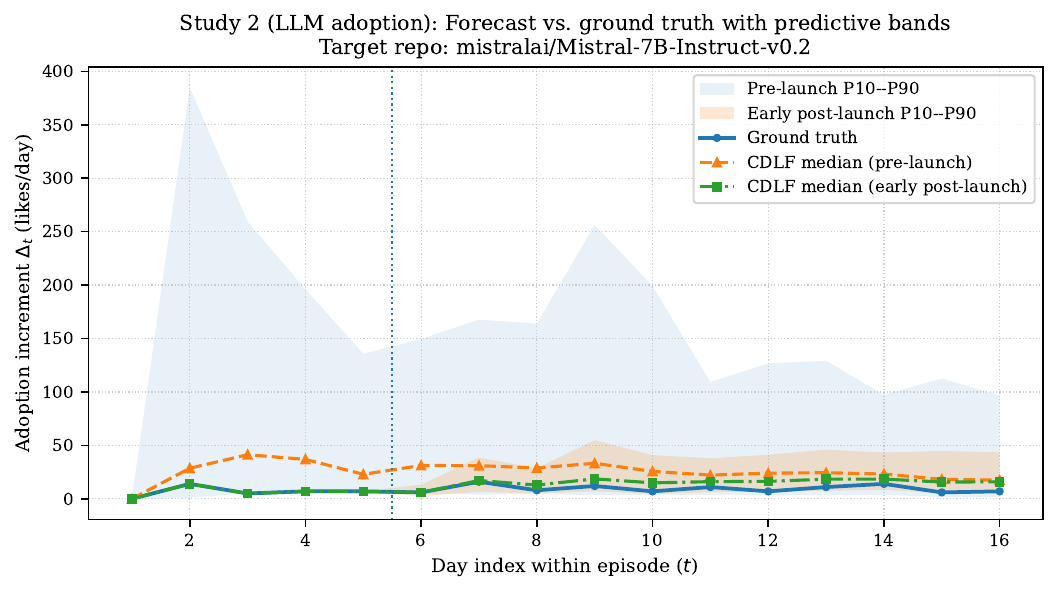}
  \caption{Study~2 (LLM adoption): forecast versus ground truth in a representative episode. The solid line shows realized adoption increments $\Delta_t$ (likes/day). CDLF is shown under two cold-start settings: pre-launch ($t_0=1$) and early post-launch ($t_0=6$), with medians distinguished by marker and line style and P10--P90 predictive bands. The vertical dotted line marks $t_0$.}
  \label{fig:hf_case_combined}
\end{figure}
\textbf{Ablation on conditioning sources.}
To isolate the contribution of each conditioning source, we conduct one-at-a-time ablations under the same episode-aligned evaluation protocol. The full CDLF conditions on the early adoption signal $\mathbf{x}_{1:t_0-1}$, static descriptors $\mathbf{s}$, and the analogical reference set $\mathcal{X}$. We remove either $\mathcal{X}$ or $\mathbf{s}$ while keeping model capacity and the training protocol unchanged. Table~\ref{tab:hf_ablation_all} shows that removing $\mathcal{X}$ yields the largest degradation across all metrics, with MAE increasing by 13.2\%, RMSE by 11.3\%, MCRPS by 12.1\%, and DTW by 14.8\%. This indicates that transferable priors from analog adoption trajectories are especially valuable when the target history is short. Removing $\mathbf{s}$ leads to a smaller but still systematic deterioration, suggesting that compact metadata plays a complementary role by improving analog selection and calibration beyond the early signal alone. The ablation results also admit a substantive interpretation. In this setting, early observations are often too sparse to reveal whether a repository will remain niche or break out, so the analogical reference set provides the main transferable life-cycle signal, while static descriptors play a smaller but still useful role by improving how that transfer is targeted across coarse product characteristics such as organizational visibility and access regime.
\begin{table}[htp]
\centering
\small
\caption{Study~2 (LLM adoption): ablation on conditioning sources. All metrics are computed on the forecast horizon $[t_0,T]$
with $T=10$ and $t_0=6$ and averaged over rolling windows.}
\label{tab:hf_ablation_all}
\begin{tabular}{lrrrr}
\hline
Variant  & MAE $\downarrow$ & RMSE $\downarrow$ & MCRPS $\downarrow$ & DTW $\downarrow$ \\
\hline
\textbf{CDLF (full)} & \textbf{0.38} & \textbf{0.62} & \textbf{0.33} & \textbf{1.35} \\
w/o $\mathcal{X}$    & 0.43 (\textbf{+13.2\%}) & 0.69 (\textbf{+11.3\%}) & 0.37 (\textbf{+12.1\%}) & 1.55 (\textbf{+14.8\%}) \\
w/o $\mathbf{s}$     & 0.40 (\textbf{+5.3\%})  & 0.65 (\textbf{+4.8\%})  & 0.35 (\textbf{+6.1\%})  & 1.42 (\textbf{+5.2\%}) \\
\hline
\end{tabular}

\vspace{2mm}
\footnotesize
\textit{Notes.} ``w/o'' removes the specified conditioning source while keeping other inputs, model capacity, and the training/tuning protocol unchanged.
The early adoption signal $\mathbf{x}_{1:t_0-1}$ is always retained; ablations only remove $\mathcal{X}$ (cross-repository reference set) or
$\mathbf{s}$ (static attributes). 
\end{table}

\textbf{Feature-conditioned pre-launch generation.}
Figure~\ref{fig:hf_feature_4profiles} illustrates what CDLF learns in the pre-launch phase. Conditioning only on the compact static profile $\mathbf{s}=(\mathrm{Org}, \mathrm{ScaleBucket}, \mathrm{Gated})$ and analog reference trajectories, CDLF generates a predictive distribution for daily adoption increments. This directly addresses a central launch-stage question: before any repository-specific adoption signal is observed, how should expected adoption intensity and uncertainty vary with a repository's positioning features?

Two broad patterns emerge. First, $\mathrm{Org}$ acts as a proxy for brand- and ecosystem-driven exposure. Repositories associated with organizations that have broader platform visibility tend to exhibit higher median adoption paths even before any early signal is observed. Second, access friction matters. Moving from open to gated access lowers the median path and widens the predictive band, indicating both weaker typical adoption and greater uncertainty. This pattern is consistent with the idea that adoption on digital platforms depends not only on intrinsic interest, but also on how easily users can access and experiment with a release. The effect of friction is especially pronounced for repositories with weaker exposure: when visibility is limited, reducing access and deployment friction appears to matter both for raising the expected path and for reducing uncertainty.

\begin{figure}[htbp]
\centering
\includegraphics[width=0.6\linewidth]{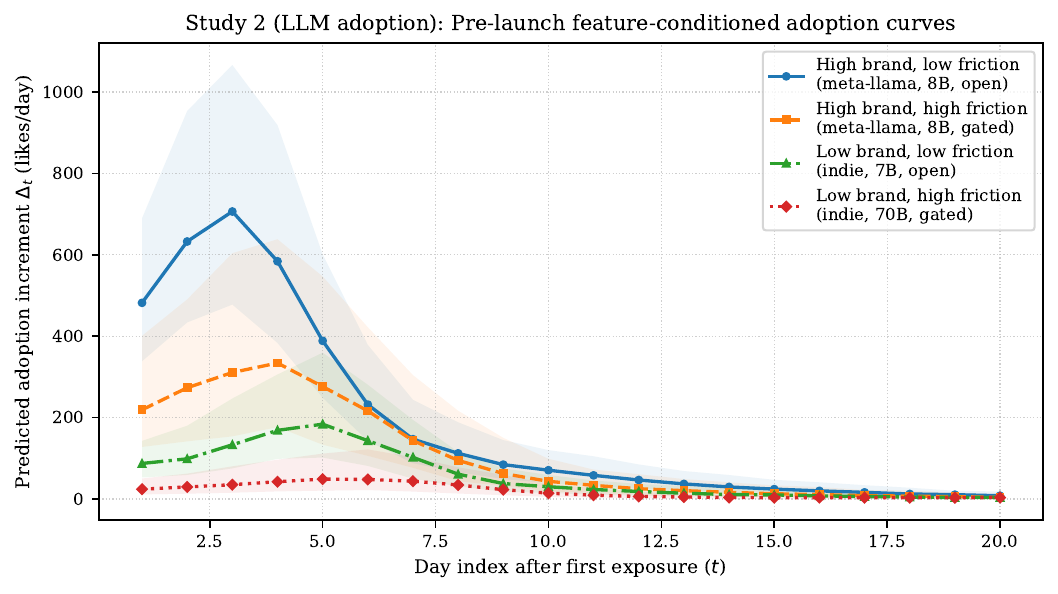}
\caption{Study~2 (LLM adoption): pre-launch adoption curves conditioned on $\mathbf{s} = (\mathrm{Org}, \mathrm{ScaleBucket}, \mathrm{Gated})$. Curves show medians with P10--P90 bands, contrasting high vs. low exposure and friction.}
\label{fig:hf_feature_4profiles}
\end{figure}

These feature-conditioned forecasts are especially meaningful in the current AI environment. Many models are released in rapid succession, but they differ not only in technical scale, but also in organizational visibility, access restrictions, and deployability. A pre-launch model that can map these differences into a predictive distribution is useful because it helps characterize which releases are likely to remain modest, which may break out, and where uncertainty is greatest even before any repository-specific prefix becomes available.

\textbf{From predictive distributions to launch segmentation.}
Beyond point forecasting, the generative output of CDLF yields interpretable summaries for launch-stage reasoning. For each repository, we summarize the predictive distribution over a fixed horizon using three statistics: (i) typical adoption load, measured by median cumulative adoption (P50 AUC); (ii) breakout potential, measured by upper-quantile peak intensity (P90 Peak); and (iii) uncertainty, measured by the inter-quantile bandwidth of cumulative adoption (P90--P10 AUC). Figure~\ref{fig:dash_dashboard_and_segmentation_A1C4} visualizes these summaries. Panel (a) plots typical adoption load against breakout potential, with marker size indicating uncertainty. Panel (b) maps repositories into a 2$\times$2 potential--risk plane based on median splits of P50 AUC and AUC bandwidth. This segmentation translates trajectory distributions into interpretable launch profiles. High-potential, high-risk releases require the greatest attention because they combine upside with substantial uncertainty. High-potential, low-risk releases appear more scalable and support broader distribution. Low-potential, high-risk releases call for de-risking, often through reduced friction or more targeted release design. Low-potential, low-risk releases can often be managed with a lighter rollout strategy. Study~2 thus shows not only that CDLF improves statistical forecasting performance, but also that its predictive distributions can be organized into structured launch profiles that are more informative than a single average trajectory.

\begin{figure}[htbp]
\centering
\subfigure[Dashboard: P50 AUC vs. P90 Peak]{
\includegraphics[height=1.7in]{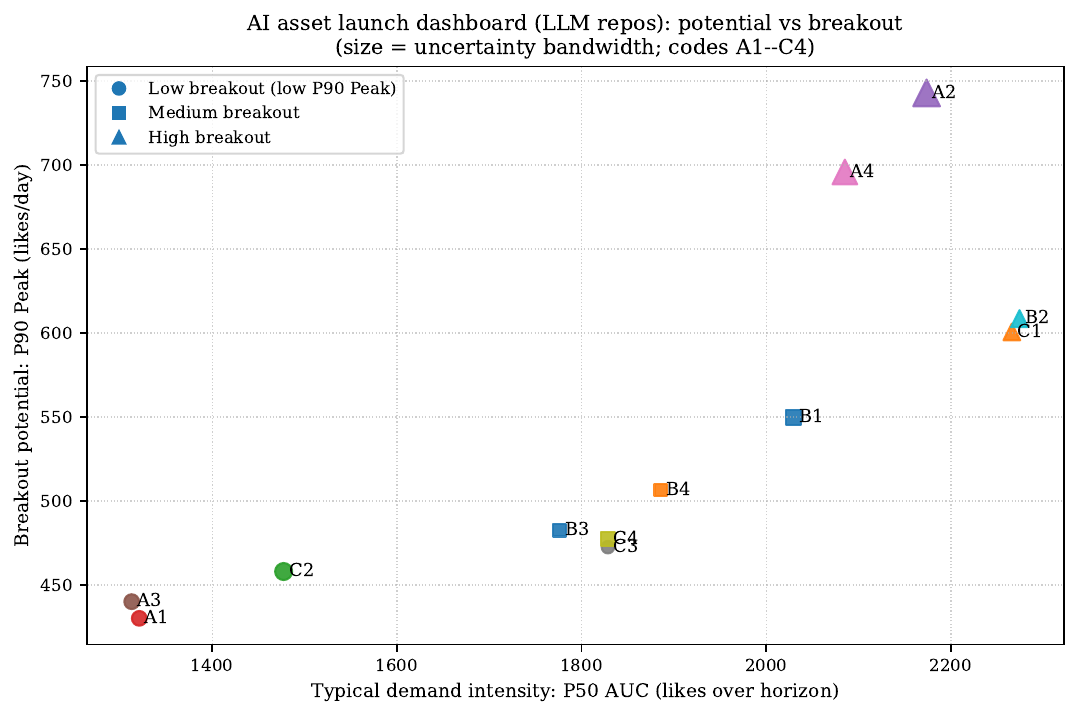}
}
\subfigure[Segmentation: Potential vs. Risk (Median Splits)]{
\includegraphics[height=1.7in]{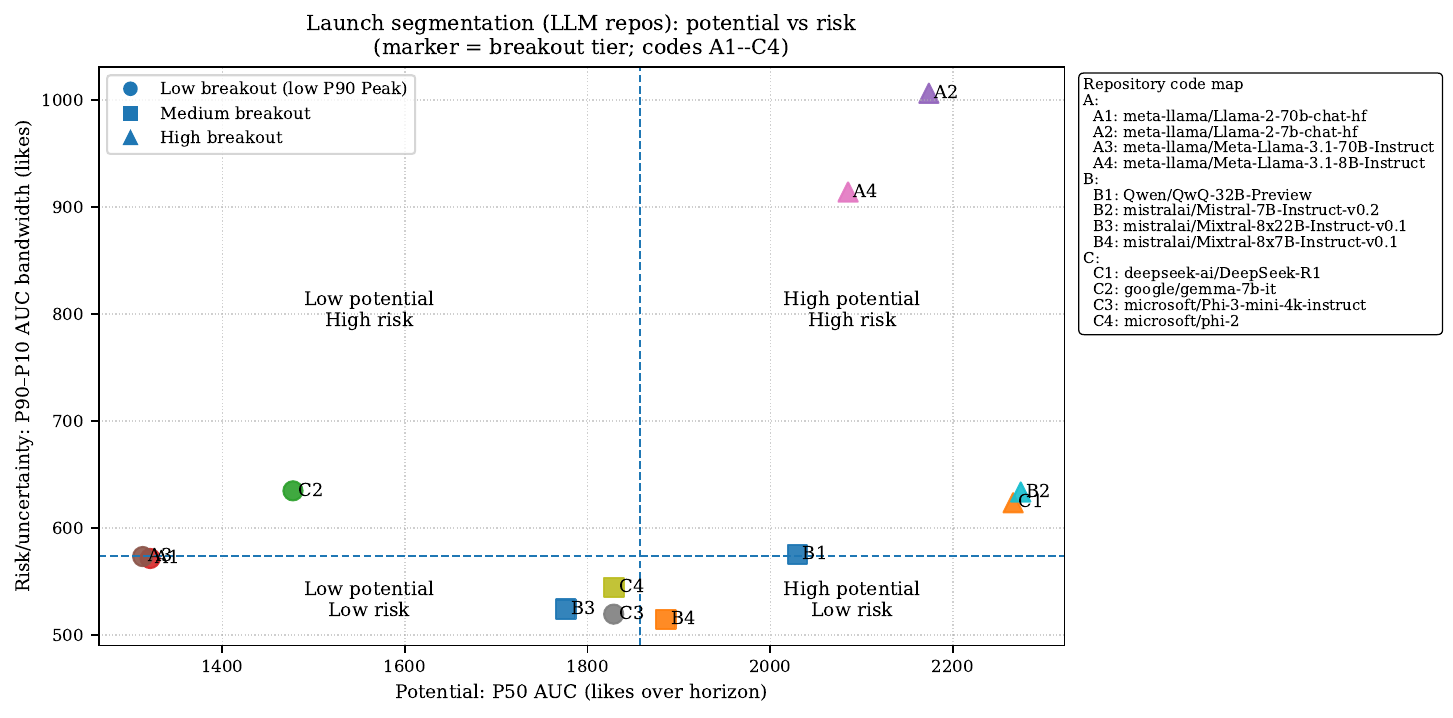}
}
\caption{Study~2 (LLM adoption): (a) typical adoption load versus breakout potential, with point size indicating uncertainty (P90--P10 AUC); (b) potential--risk segmentation for LLM repositories based on median splits of P50 AUC and AUC bandwidth.}
\label{fig:dash_dashboard_and_segmentation_A1C4}
\end{figure}

Figure~\ref{fig:dash_playbook_A1C4} translates the resulting 2$\times$2 taxonomy into an actionable playbook. High-potential, high-risk releases require greater emphasis on capacity, support, and burst capture. High-potential, low-risk releases are better suited to broad distribution with low friction. Low-potential, high-risk releases should prioritize de-risking by reducing access or deployment barriers. Low-potential, low-risk releases can support a lightweight rollout and community seeding strategy.

\begin{figure}[htbp]
\centering
\includegraphics[width=0.7\linewidth]{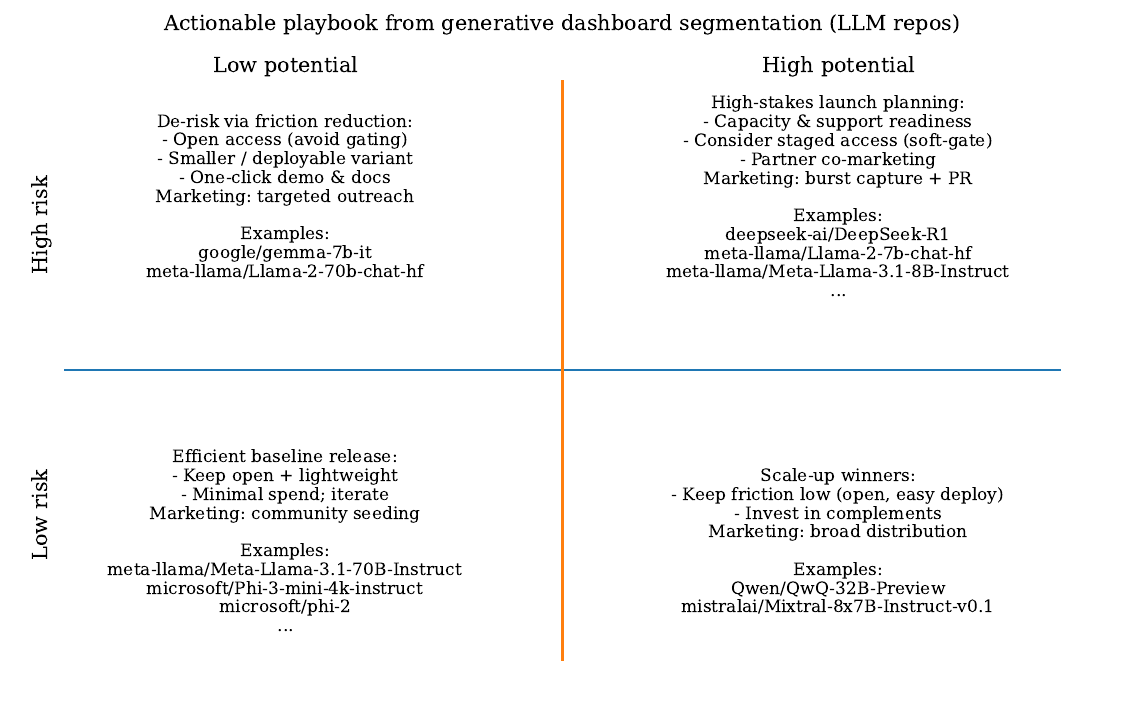}
\caption{Actionable playbook from generative launch segmentation. Each quadrant maps potential and risk segments to release actions, with candidate repositories listed in each quadrant.}
\label{fig:dash_playbook_A1C4}
\end{figure}

\textbf{Qualitative triangulation with public release rationales.}
We further triangulate our dashboard segmentation with publicly stated rollout rationales of the six teams behind the 12 repositories, namely Meta/Llama, DeepSeek, Google/Gemma, Microsoft/Phi, Mistral/Mixtral, and Qwen/QwQ. Across teams, two recurrent themes align closely with the CDLF-based summaries: exposure and friction. High-visibility releases are often paired with broad platform distribution and staged channel rollout, which is consistent with our high-potential segments where capacity, support, and burst capture are most salient.\footnote{See public availability announcements and model-catalog listings for Llama-family deployments on major channels, including AWS Bedrock ``What's New,'' the Microsoft Azure AI model catalog, and Google Cloud Vertex AI partner-model documentation. Retrieved January 13, 2026: \url{https://aws.amazon.com/about-aws/whats-new/}; \url{https://ai.azure.com/catalog/models/}; \url{https://cloud.google.com/vertex-ai}.}

By contrast, gated releases that require users to accept license or usage terms before download directly instantiate the friction mechanism highlighted by the feature-conditioned forecasts and the segmentation results.\footnote{On the Hugging Face Hub, gated access is indicated on model pages and typically requires accepting license or usage terms prior to file access. See Hugging Face documentation and example model pages, such as \texttt{meta-llama/Llama-2-7b-chat-hf} and \texttt{google/gemma-7b-it}. Retrieved January 13, 2026: \url{https://huggingface.co/docs}; \url{https://huggingface.co/meta-llama/Llama-2-7b-chat-hf}; \url{https://huggingface.co/google/gemma-7b-it}.} Efficiency- and deployability-oriented strategies, including sparse mixture-of-experts designs and smaller models, likewise align with the idea that easier deployment can support broader and more stable diffusion.\footnote{See the official release notes and positioning statements for Mixtral and Phi model families, including Mistral AI news releases and Microsoft Azure AI / TechCommunity materials. Retrieved January 13, 2026: \url{https://mistral.ai/news/}; \url{https://techcommunity.microsoft.com/}.} Preview-style releases are also consistent with staged rollout under elevated uncertainty.\footnote{Example: \texttt{Qwen/QwQ-32B-Preview} model page on the Hugging Face Hub. Retrieved January 13, 2026: \url{https://huggingface.co/Qwen/QwQ-32B-Preview}.} Taken together, these qualitative patterns do not establish causality, but they provide a useful triangulation: the main dimensions recovered by CDLF are also visible in how leading teams publicly frame their own release strategies.
\section{Concluding Remarks}\label{sec6}

This paper proposes CDLF, a unified generative framework for cold-start life-cycle forecasting in the pre-launch and early post-launch phases. Across both empirical studies, CDLF improves upon classical diffusion models, Bayesian updating approaches, and machine-learning baselines in both point accuracy and probabilistic forecast quality over short and medium horizons. More broadly, the results show that conditional generative forecasting can improve early-stage uncertainty quantification while preserving coherent dependence across future periods. 

From a managerial perspective, CDLF provides a unified forecasting mechanism across information regimes. It supports prediction before launch, when no product-specific observations are available, and then updates those forecasts as early evidence emerges. This feature is useful because it allows forecasts to be revised in a timely and disciplined way, thereby supporting better adjustment of launch, inventory, and capacity plans under severe uncertainty.

Several directions for future research follow naturally. One is to incorporate richer contextual signals and interventions, including marketing actions, search trends, platform exposure, and other demand-shaping variables. A second is to extend the framework to multivariate trajectory generation so that multiple operationally relevant outcomes can be forecast jointly. A third is to connect the predictive distributions generated by CDLF to downstream planning problems, such as launch-stage capacity planning, compute deployment, support provisioning, and staged rollout decisions under uncertainty.

% Acknowledgments here
\ACKNOWLEDGMENT{%
% Enter the text of acknowledgments here
}% Leave this (end of acknowledgment)

\bibliographystyle{informs2014} % outcomment this and next line in Case 1
%\bibliography{<your bib file$\textbf{s}$>} % if more than one, comma separated
\bibliography{name}
% Appendix here
% Options are (1) APPENDIX (with or without general title) or 
%             (2) APPENDICES (if it has more than one unrelated sections)
% Outcomment the appropriate case if necessary
%
%\begin{APPENDIX}{<Title of the Appendix>}
% \end{APPENDIX}
%
%   or 
%

\ECSwitch

\ECHead{\centering Online Appendix to\\
``Cold-Start Forecasting of New Product Life-Cycles via Conditional Diffusion Models''}

\makeatletter
% section / subsection
\setcounter{section}{0}
\setcounter{subsection}{0}
\setcounter{subsubsection}{0}
\renewcommand{\thesection}{\Alph{section}}
\renewcommand{\thesubsection}{\thesection.\arabic{subsection}}
\renewcommand{\thesubsubsection}{\thesubsection.\arabic{subsubsection}}

% equation
\setcounter{equation}{0}
\renewcommand{\theequation}{\thesection.\arabic{equation}}
\@addtoreset{equation}{section}

% figure / table
\setcounter{figure}{0}
\setcounter{table}{0}
\renewcommand{\thefigure}{\thesection.\arabic{figure}}
\renewcommand{\thetable}{\thesection.\arabic{table}}
\@addtoreset{figure}{section}
\@addtoreset{table}{section}

% hyperref anchors
\renewcommand{\theHsection}{appendix.\Alph{section}}
\renewcommand{\theHsubsection}{appendix.\Alph{section}.\arabic{subsection}}
\renewcommand{\theHsubsubsection}{appendix.\Alph{section}.\arabic{subsection}.\arabic{subsubsection}}
\renewcommand{\theHequation}{appendix.\Alph{section}.\arabic{equation}}
\renewcommand{\theHfigure}{appendix.\Alph{section}.\arabic{figure}}
\renewcommand{\theHtable}{appendix.\Alph{section}.\arabic{table}}
\renewcommand{\theHlemma}{appendix.\Alph{section}.\arabic{lemma}}
\renewcommand{\theHtheorem}{appendix.\Alph{section}.\arabic{theorem}}
\renewcommand{\theHproposition}{appendix.\Alph{section}.\arabic{proposition}}
\makeatother

This online appendix provides supplementary material for the main text. It first reports additional empirical evidence on the similarity-weight fusion design used in the contextual encoder, complementing the methodological discussion in Section~\ref{sec31}. It then presents the detailed proof of the main forecasting error bound and consistency result, together with the intermediate lemmas used in the argument. Next, it collects the technical conditions underlying the theoretical analysis, including the construction of the reachable set and the moment/support conditions required for the Wasserstein-based error metrics to be well defined. The appendix also provides a numerical illustration of the stability margin, along with a quantitative sensitivity analysis, to connect the abstract multi-step error bound with observable forecasting behavior under recursive generation. Finally, it presents implementation-oriented sufficient conditions that make the stability requirement for the GRU state-transition model checkable in practice.
\section{Ablation on Similarity-Weight Fusion}
\label{app:weight_fusion_2ways}

This appendix section supplements the contextual-information encoding and state-construction discussion in Section~\ref{sec31}. In the main text, the reference-set representation is formed by combining similarity weights with encoded reference trajectories through concatenation and projection. Here we isolate that design choice and compare it with a simpler multiplicative weighting rule. The purpose is to clarify why the proposed fusion operator is retained in the main model and to show that its empirical advantage is not driven by other components of the forecasting architecture.

We use bold lowercase letters for vectors and bold uppercase letters for matrices. The target multivariate time series is denoted by $\mathbf{x}_{1:T}=\{\mathbf{x}_t\}_{t=1}^{T}$, where $\mathbf{x}_t\in\mathbb{R}^{D}$ is the $D$-dimensional observation vector at time step $t$. The reference set $\mathcal{X}=\{\mathbf{x}^{(k)}_{1:T}\}_{k=1}^{K}$ contains $K$ reference trajectories. Each reference trajectory is encoded as $\mathbf{h}_k\in\mathbb{R}^{d}$, the static-descriptor embedding is $\mathbf{h}_s\in\mathbb{R}^{d_s}$, and the aggregated reference-set representation is $\mathbf{h}_{\mathcal{X}}\in\mathbb{R}^{d}$. The similarity weight satisfies $\omega_k\in(0,1)$ and $\sum_{k=1}^{K}\omega_k=1$. Operators and functions are described as follows.
\begin{itemize}
    \item \textbf{Concatenation}: $[\mathbf{a};\mathbf{b}]$ concatenates vectors along the feature dimension. If $\mathbf{a}\in\mathbb{R}^{m}$ and $\mathbf{b}\in\mathbb{R}^{n}$, then $[\mathbf{a};\mathbf{b}]\in\mathbb{R}^{m+n}$. When $\mathbf{b}$ is a scalar (e.g., $\omega_k$), $[\mathbf{h}_k;\omega_k]\in\mathbb{R}^{d+1}$.
    \item \textbf{Linear map}: $\mathbf{W}_\omega$. Here $\mathbf{W}_{\omega}\in\mathbb{R}^{d\times(d+1)}$ projects a $(d{+}1)$-dimensional vector back to $d$ dimensions.
    \item \textbf{ReLU}: $\mathrm{ReLU}(z)=\max(0,z)$, applied element-wise.
    \item \textbf{Softmax}: $\mathrm{softmax}(a_k)=\exp(a_k)/\sum_j\exp(a_j)$.
    \item \textbf{Ordering convention for set-to-sequence}: since $\mathrm{RNN}_{\text{ref}}(\cdot)$ takes a sequence input, we feed references sorted by $\omega_k$ (or similarity scores) in descending order to avoid order ambiguity.
\end{itemize}

Study~1 uses the Intel microprocessor weekly sales dataset (86 CPU SKUs, up to 187 weeks), where each SKU is described by static descriptors (e.g., generation, core count, clock speed, and average selling price). We follow the Study~1 preprocessing and evaluation setup, including pre-launch forecasting (static descriptors plus references only) and early post-launch forecasting (incrementally incorporating $\mathbf{x}_{1:t_0-1}$). For each reference trajectory $\mathbf{x}^{(k)}_{1:T}$, we obtain $\mathbf{h}_k=\psi(\mathbf{x}^{(k)}_{1:T})$ using a shared encoder. Given normalized similarity weights $\omega_k$ computed from static descriptors $\mathbf{s}$, we compare:
(A) multiplicative scaling: $\tilde{\mathbf{h}}_k=\omega_k\mathbf{h}_k$;
(B) concatenation plus projection: $\tilde{\mathbf{h}}_k=\mathrm{ReLU}\!\left(\mathbf{W}_{\omega}[\mathbf{h}_k;\omega_k]+\mathbf{b}_{\omega}\right)$. We then apply the same aggregator $\mathbf{h}_{\mathcal{X}}=\mathrm{RNN}_{\text{ref}}(\{\tilde{\mathbf{h}}_1,\tilde{\mathbf{h}}_2,\dots,\tilde{\mathbf{h}}_K\})$.
To avoid order ambiguity, references are fed into $\mathrm{RNN}_{\text{ref}}$ sorted by $\omega_k$ in descending order.

We follow the exact Study~1 data split, training schedule, hyperparameters, and random seed list. Both variants use the same reference retrieval procedure and reference set size $K$, share the same $\psi(\cdot)$, $\mathrm{RNN}_{\text{ref}}$, and generative/forecasting backbone, and differ only in the fusion operator. We report the same Study~1 metrics: MAE and MCRPS (both scaled by $\times 10^3$) for the 1--8 and 9--16 weeks-ahead horizons.
\begin{table}[htp]
\small
\centering
\caption{Study~1: ablation on similarity-weight fusion (lower is better; $\times 10^3$ scaling).}
\label{tab:ablation_weight_fusion_2ways}
\begin{tabular}{lcccc}
\toprule
\multirow[c]{2}{*}{\centering Fusion method}
& \multicolumn{2}{c}{1--8 weeks ahead}
& \multicolumn{2}{c}{9--16 weeks ahead} \\
\cmidrule(lr){2-3}\cmidrule(lr){4-5}
& MAE $\downarrow$ & MCRPS $\downarrow$ & MAE $\downarrow$ & MCRPS $\downarrow$ \\
\midrule
(A) Multiplicative scaling: $\tilde{\mathbf{h}}_k=\omega_k\mathbf{h}_k$
& 14.86 & 10.41 & 20.28 & 15.43 \\
(B) Concat + proj (ReLU): $\tilde{\mathbf{h}}_k=\mathrm{ReLU}\!\left(\mathbf{W}_{\omega}[\mathbf{h}_k;\omega_k]+\mathbf{b}_{\omega}\right)$
& 14.52 & 10.16 & 19.96 & 15.12 \\
\bottomrule
\end{tabular}
\end{table}
In Study~1, $\omega_k$ is computed from static descriptors $\mathbf{s}$ and can be viewed as a relevance prior, but static similarity does not necessarily imply perfect similarity in sales-curve shapes. The multiplicative baseline $\omega_k\mathbf{h}_k$ imposes a rigid global attenuation and may suppress informative references when $\omega_k$ is noisy or the softmax distribution becomes overly sharp. By contrast, concatenation with a learnable projection treats $\omega_k$ as a conditioning feature and learns how similarity should modulate reference content through a nonlinear mapping, which is typically more robust and better preserves moderately similar but pattern-critical references.
\section{Proof of Theorem~\ref{thm:multi_step_consistency_main}}

This appendix section provides the detailed proof of Theorem~\ref{thm:multi_step_consistency_main}, which establishes the multi-step forecasting error bound and consistency result in the main text. The proof proceeds in two steps. We first establish auxiliary lemmas that separate one-step generative error from latent-state propagation error and characterize the recursion of the latent-state mismatch. We then combine these ingredients to derive the geometric multi-step bound under the stability margin condition.

\subsection{Lemmas for the Proof of Theorem~\ref{thm:multi_step_consistency_main}}

This subsection collects the two intermediate lemmas used in the proof. The first controls how latent-state error propagates through the recurrent transition, and the second decomposes the one-step forecasting error into a generative term and a propagated state-mismatch term.

\begin{lemma}[Latent-state error recursion]
    \label{lem:hidden_recursion_main}
    Suppose that Assumption~\ref{assump:combined_main} holds.
    Define the \emph{transition error} relative to the best-in-class transition $f^\star$ as
    \begin{equation}
    \varepsilon_f
    :=
    \sup_{t}\ 
    \mathbb{E}\Big[
    \big\|
    f_\phi(\mathbf{h}_{t-1}^\star,\mathbf{x}_t,\mathbf{c})
    -
    f^\star(\mathbf{h}_{t-1}^\star,\mathbf{x}_t,\mathbf{c})
    \big\|
    \Big].
    \end{equation}
    Then the latent-state error satisfies
    \begin{equation}
    \label{eq:hidden_recursion_main}
    E_t
    \le
    \rho E_{t-1}
    +
    L_x\,\mathbb{E}\big[\|\widehat{\mathbf{x}}_t-\mathbf{x}_t\|\big]
    +
    \varepsilon_f.
    \end{equation}
\end{lemma}

Lemma~\ref{lem:hidden_recursion_main} shows that the latent-state error propagates through three channels: the contracted previous mismatch (factor $\rho$), the propagated input mismatch (factor $L_x$), and the transition error $\varepsilon_f$. The term $\varepsilon_f$ measures how well the learned transition network $f_\phi$ approximates the best-in-class transition $f^\star$ along the best-in-class trajectory. In general, $\varepsilon_f$ can be further decomposed into an approximation error (due to the finite capacity of the transition architecture) and a statistical error (due to finite training samples). Accordingly, as the training sample size increases and the transition architecture becomes more expressive, $\varepsilon_f$ decreases and can vanish asymptotically.

\noindent\textit{Proof of Lemma~\ref{lem:hidden_recursion_main}.}\quad
From the learned dynamics and the best-in-class dynamics,
\[
\widehat{\mathbf{h}}_t=f_\phi(\widehat{\mathbf{h}}_{t-1},\widehat{\mathbf{x}}_t,\mathbf{c}),
\qquad
\mathbf{h}_t^\star=f^\star(\mathbf{h}_{t-1}^\star,\mathbf{x}_t,\mathbf{c}).
\]
Add and subtract $f_\phi(\mathbf{h}_{t-1}^\star,\mathbf{x}_t,\mathbf{c})$, then apply the triangle inequality:
\begin{align}
\| \widehat{\mathbf{h}}_t-\mathbf{h}_t^\star\|
&\le
\big\|f_\phi(\widehat{\mathbf{h}}_{t-1},\widehat{\mathbf{x}}_t,\mathbf{c})
      -f_\phi(\mathbf{h}_{t-1}^\star,\mathbf{x}_t,\mathbf{c})\big\|
+\big\|f_\phi(\mathbf{h}_{t-1}^\star,\mathbf{x}_t,\mathbf{c})
      -f^\star(\mathbf{h}_{t-1}^\star,\mathbf{x}_t,\mathbf{c})\big\|.
\label{eq:hidden_recursion_step1_app}
\end{align}
Apply Assumption~\ref{assump:combined_main} with
$(\mathbf{h},\mathbf{x})=(\widehat{\mathbf{h}}_{t-1},\widehat{\mathbf{x}}_t)$ and
$(\mathbf{h}',\mathbf{x}')=(\mathbf{h}_{t-1}^\star,\mathbf{x}_t)$ to bound the first term:
\begin{equation}
\label{eq:hidden_recursion_step2_app}
\big\|f_\phi(\widehat{\mathbf{h}}_{t-1},\widehat{\mathbf{x}}_t,\mathbf{c})
      -f_\phi(\mathbf{h}_{t-1}^\star,\mathbf{x}_t,\mathbf{c})\big\|
\le
\rho\|\widehat{\mathbf{h}}_{t-1}-\mathbf{h}_{t-1}^\star\|
+L_x\|\widehat{\mathbf{x}}_t-\mathbf{x}_t\|.
\end{equation}
Substituting \eqref{eq:hidden_recursion_step2_app} into \eqref{eq:hidden_recursion_step1_app} yields
\[
\|\widehat{\mathbf{h}}_t-\mathbf{h}_t^\star\|
\le
\rho\|\widehat{\mathbf{h}}_{t-1}-\mathbf{h}_{t-1}^\star\|
+L_x\|\widehat{\mathbf{x}}_t-\mathbf{x}_t\|
+\big\|f_\phi(\mathbf{h}_{t-1}^\star,\mathbf{x}_t,\mathbf{c})
      -f^\star(\mathbf{h}_{t-1}^\star,\mathbf{x}_t,\mathbf{c})\big\|.
\]
Taking expectations and using $E_t:=\mathbb{E}\|\widehat{\mathbf{h}}_t-\mathbf{h}_t^\star\|$ gives
\begin{align*}
E_t
&\le
\rho\,\mathbb{E}\|\widehat{\mathbf{h}}_{t-1}-\mathbf{h}_{t-1}^\star\|
+L_x\,\mathbb{E}\|\widehat{\mathbf{x}}_t-\mathbf{x}_t\|
+\mathbb{E}\big\|f_\phi(\mathbf{h}_{t-1}^\star,\mathbf{x}_t,\mathbf{c})
      -f^\star(\mathbf{h}_{t-1}^\star,\mathbf{x}_t,\mathbf{c})\big\| \\
&=
\rho E_{t-1}
+L_x\,\mathbb{E}\|\widehat{\mathbf{x}}_t-\mathbf{x}_t\|
+\mathbb{E}\big\|f_\phi(\mathbf{h}_{t-1}^\star,\mathbf{x}_t,\mathbf{c})
      -f^\star(\mathbf{h}_{t-1}^\star,\mathbf{x}_t,\mathbf{c})\big\|.
\end{align*}
By the definition of $\varepsilon_f$, the last expectation is bounded by $\varepsilon_f$, which yields \eqref{eq:hidden_recursion_main}.
\hfill$\Box$

\vspace{0.5em}

\begin{lemma}[One-step prediction error decomposition]
    \label{lem:one_step_decomp_main}
    Under Assumption~\ref{assump:combined_main}, define $\varepsilon_{\mathrm{gen}}$ via the uniform one-step generative bound on $\mathcal{H}$: for all $t$ and all $\mathbf{h}\in\mathcal{H}$,
    \begin{equation}
    \label{eq:gen_uniform_bound_main}
    W_1\!\left(p_\theta(\cdot\mid \mathbf{h}),\,p^\star(\cdot\mid \mathbf{h})\right)\le \varepsilon_{\mathrm{gen}}.
    \end{equation}
    Then for any $t$,
    \begin{equation}
    \label{eq:one_step_decomp_main}
    W_1\!\left(p_\theta(\cdot\mid \widehat{\mathbf{h}}_{t-1}),\,p^\star(\cdot\mid \mathbf{h}_{t-1}^\star)\right)
    \le \varepsilon_{\mathrm{gen}} + L_P\|\widehat{\mathbf{h}}_{t-1}-\mathbf{h}_{t-1}^\star\|.
    \end{equation}
    Consequently,
    \begin{equation}
    \label{eq:one_step_decomp_expect_main}
    \Delta_t \le \varepsilon_{\mathrm{gen}} + L_P E_{t-1}.
    \end{equation}
\end{lemma}

Lemma~\ref{lem:one_step_decomp_main} decomposes the one-step forecasting error into two components: (i) the generative error $\varepsilon_{\mathrm{gen}}$, which measures how well the learned conditional generator $p_\theta(\cdot\mid \mathbf{h})$ approximates the best-in-class generator $p^\star(\cdot\mid \mathbf{h})$ uniformly over the reachable set $\mathcal{H}$, and (ii) the propagated latent-state error $E_{t-1}$.

\noindent\textit{Proof of Lemma~\ref{lem:one_step_decomp_main}.}\quad
Fix any $t$ and define
\[
P := p_\theta(\cdot\mid \widehat{\mathbf{h}}_{t-1}),\qquad
R := p_t^\star(\cdot\mid \widehat{\mathbf{h}}_{t-1}),\qquad
Q := p_t^\star(\cdot\mid \mathbf{h}_{t-1}^\star).
\]
By Assumption~\ref{assump:moment_support_main}, $\widehat{\mathbf{h}}_{t-1},\mathbf{h}_{t-1}^\star\in\mathcal{H}$ a.s. and all three laws belong to $\mathcal{P}_1(\mathbb{R}^d)$. Since $W_1$ is a metric on $\mathcal{P}_1(\mathbb{R}^d)$, it satisfies the triangle inequality:
\[
W_1(P,Q)\le W_1(P,R)+W_1(R,Q).
\]
Because $\widehat{\mathbf{h}}_{t-1}\in\mathcal{H}$ a.s., substituting $\mathbf{h}=\widehat{\mathbf{h}}_{t-1}$ into \eqref{eq:gen_uniform_bound_main} yields
\[
W_1(P,R)=W_1\!\left(p_\theta(\cdot\mid \widehat{\mathbf{h}}_{t-1}),\,p_t^\star(\cdot\mid \widehat{\mathbf{h}}_{t-1})\right)\le \varepsilon_{\mathrm{gen}}
\quad\text{a.s.}
\]
Because $\widehat{\mathbf{h}}_{t-1},\mathbf{h}_{t-1}^\star\in\mathcal{H}$ a.s., Assumption~\ref{assump:combined_main} gives
\[
W_1(R,Q)=W_1\!\left(p_t^\star(\cdot\mid \widehat{\mathbf{h}}_{t-1}),\,p_t^\star(\cdot\mid \mathbf{h}_{t-1}^\star)\right)
\le L_P\|\widehat{\mathbf{h}}_{t-1}-\mathbf{h}_{t-1}^\star\|
\quad\text{a.s.}
\]
Combining the two bounds yields the pathwise inequality \eqref{eq:one_step_decomp_main}. Taking expectations on both sides and using linearity of expectation gives
\[
\Delta_t=\mathbb{E}\Big[W_1(P,Q)\Big]
\le \varepsilon_{\mathrm{gen}} + L_P\,\mathbb{E}\|\widehat{\mathbf{h}}_{t-1}-\mathbf{h}_{t-1}^\star\|
= \varepsilon_{\mathrm{gen}} + L_P E_{t-1},
\]
which is \eqref{eq:one_step_decomp_expect_main}.
\hfill$\Box$

\subsection{Proof of Theorem~\ref{thm:multi_step_consistency_main}}
\label{app:proof_uniform_bound}

With the preceding lemmas in place, we now prove Theorem~\ref{thm:multi_step_consistency_main} by linking one-step conditional distribution error to latent-state recursion and then controlling the resulting geometric accumulation over the forecasting horizon.

\noindent\textit{Proof of Theorem~\ref{thm:multi_step_consistency_main}.}\quad
Note that
$p_t^\star(\cdot\mid \mathcal{F}_{t-1}) \equiv p_t^\star(\cdot\mid \mathbf{h}_{t-1}^\star)$.
Lemma~\ref{lem:one_step_decomp_main} yields, for any $t\ge t_0$,
\begin{equation}
\label{eq:Delta_vs_E_thm_refined}
\Delta_t \le \varepsilon_{\mathrm{gen}} + L_P E_{t-1}.
\end{equation}
Lemma~\ref{lem:hidden_recursion_main} yields, for any $t\ge t_0$,
\begin{equation}
\label{eq:E_recursion_thm_refined}
E_{t-1}\le \rho E_{t-2} + L_x\,\mathbb{E}\|\widehat{\mathbf{x}}_{t-1}-\mathbf{x}_{t-1}\| + \varepsilon_{f}.
\end{equation}

It remains to control $\mathbb{E}\|\widehat{\mathbf{x}}_{t-1}-\mathbf{x}_{t-1}\|$ in terms of $\Delta_{t-1}$.
Let $\mathcal{G}_{t-2}:=\sigma(\widehat{\mathbf{h}}_{t-2},\mathbf{h}_{t-2}^\star)$ and define the $\mathcal{G}_{t-2}$-measurable random probability measures
\[
\widehat{P}_{t-1}:=p_\theta(\cdot\mid \widehat{\mathbf{h}}_{t-2}),
\qquad
P_{t-1}^\star:=p_{t-1}^\star(\cdot\mid \mathbf{h}_{t-2}^\star).
\]
By the definition of $W_1$ as an infimum over couplings, for any $\eta>0$ and for each realization of $(\widehat{P}_{t-1},P_{t-1}^\star)$ there exists a coupling $\pi_{t-1}^{(\eta)}\in\Pi(\widehat{P}_{t-1},P_{t-1}^\star)$ such that
\[
\int \|\widehat{\mathbf{x}}-\mathbf{x}\|\,\pi_{t-1}^{(\eta)}(d\widehat{\mathbf{x}},d\mathbf{x})
\le W_1(\widehat{P}_{t-1},P_{t-1}^\star)+\eta.
\]
Because $\mathbb{R}^d$ is Polish and the transport cost $\|\widehat{\mathbf{x}}-\mathbf{x}\|$ is continuous, standard measurable-selection results for optimal transport ensure that such $\eta$-optimal couplings can be chosen as a measurable coupling kernel of $(\widehat{P}_{t-1},P_{t-1}^\star)$; hence we may sample a pair $(\widehat{\mathbf{X}}_{t-1}^{(\eta)},\mathbf{X}_{t-1}^{(\eta)})$ conditionally on $\mathcal{G}_{t-2}$ from $\pi_{t-1}^{(\eta)}$.
Taking conditional expectations and then total expectations yields
\begin{align}
\mathbb{E}\big\|\widehat{\mathbf{X}}_{t-1}^{(\eta)}-\mathbf{X}_{t-1}^{(\eta)}\big\|
&\le
\mathbb{E}\Big[W_1\!\left(p_\theta(\cdot\mid \widehat{\mathbf{h}}_{t-2}),\,p_{t-1}^\star(\cdot\mid \mathbf{h}_{t-2}^\star)\right)\Big]+\eta \nonumber\\
&= \Delta_{t-1}+\eta
\le \overline{\Delta}+\eta.
\label{eq:coupling_to_Delta}
\end{align}
Letting $\eta\downarrow 0$ gives
\begin{equation}
\label{eq:input_mismatch_le_Delta}
\inf_{\pi\in \Pi(\widehat{P}_{t-1},P_{t-1}^\star)}
\mathbb{E}_{(\widehat{\mathbf{X}}_{t-1},\mathbf{X}_{t-1})\sim \pi}
\big\|\widehat{\mathbf{X}}_{t-1}-\mathbf{X}_{t-1}\big\|
\le \overline{\Delta}.
\end{equation}
The left-hand side is exactly the $1$-Wasserstein distance $W_1(\widehat{P}_{t-1},P_{t-1}^\star)$. Since \eqref{eq:E_recursion_thm_refined} depends only on an upper bound for the expected mismatch magnitude, we may bound the term $\mathbb{E}\|\widehat{\mathbf{x}}_{t-1}-\mathbf{x}_{t-1}\|$ in \eqref{eq:E_recursion_thm_refined} by \eqref{eq:input_mismatch_le_Delta}, obtaining
\[
E_{t-1}\le \rho E_{t-2} + L_x\,\overline{\Delta} + \varepsilon_{f}.
\]
Iterating from $t=t_0$ to any $t\in\{t_0,\dots,T\}$ yields
\[
E_{t-1}
\le
\rho^{t-t_0}\,E_{t_0-1}
+\sum_{j=t_0}^{t-1}\rho^{t-1-j}\Big(L_x\,\overline{\Delta}+\varepsilon_{f}\Big)
\le
E_{t_0-1} + \frac{L_x}{1-\rho}\overline{\Delta} + \frac{1}{1-\rho}\varepsilon_{f}.
\]
Substituting into \eqref{eq:Delta_vs_E_thm_refined} and taking supremum over $t\in\{t_0,\dots,T\}$ gives
\[
\overline{\Delta}
\le
\varepsilon_{\mathrm{gen}}
+
L_P\left(
\frac{L_x}{1-\rho}\overline{\Delta}
+
\frac{1}{1-\rho}\varepsilon_{f}
+
E_{t_0-1}
\right).
\]
Using $\kappa=\frac{L_P L_x}{1-\rho}<1$ to move the $\overline{\Delta}$ term to the left yields the final bound.
\hfill$\Box$

\section{Reachable Set Construction and Technical Conditions}
\label{app:reachable_set}

This appendix section collects the technical conditions used in the theoretical analysis and makes precise the domain on which the main stability assumptions are imposed. Specifically, it constructs the reachable set $\mathcal H$ for the recursive forecasting system, states the moment/support condition required for the Wasserstein-based error metrics to be well defined, and introduces a forward-invariance property that ensures these assumptions need only hold on the dynamically relevant region of the latent-state space. Together, these ingredients justify the local regularity conditions used in the proof of the main theorem.

\medskip
\subsection{Constructing the Reachable Set}

We begin by defining the set of latent states that can be visited by either the learned or the best-in-class recursion over the forecasting horizon. This construction identifies the effective state domain on which the theoretical comparison is carried out.

Fix the conditioning context $\mathbf{c}$ and the forecasting start time $t_0$.
Let $\mathcal H_{t_0-1}$ denote the set of possible initial latent states at time $t_0-1$, including both the best-in-class and learned states (e.g., the support, or a high-probability set) of $(\mathbf{h}_{t_0-1}^\star,\widehat{\mathbf{h}}_{t_0-1})$, where $\widehat{\mathbf{h}}_{t_0-1}$ encodes the available prefix history $\mathbf{x}_{1:t_0-1}$ together with the static context $\mathbf{c}$ (and $t_0=1$ corresponds to an empty prefix history, hence $t_0-1=0$).

For $t\ge t_0$, define $\mathcal H_t$ recursively as the forward-closure of $\mathcal H_{t-1}$ under both the best-in-class and learned recursions:
\[
\small
\mathcal H_t := \mathcal H_{t-1}
\ \cup\
\Big\{
f^\star(\mathbf{h},\mathbf{x},\mathbf{c})
:
\mathbf{h}\in\mathcal H_{t-1},
\mathbf{x}\in \mathrm{supp}\big(p_t^\star(\cdot\mid \mathbf{h})\big)
\Big\}
\ \cup\
\Big\{
f_\phi(\mathbf{h},\widehat{\mathbf{x}},\mathbf{c})
:
\mathbf{h}\in\mathcal H_{t-1},
\widehat{\mathbf{x}}\in \mathrm{supp}\big(p_\theta(\cdot\mid \mathbf{h})\big)
\Big\}.
\]
Equivalently, we start from all possible initial states, repeatedly apply one-step transitions under the best-in-class recursion $f^\star$ and the learned recursion $f_\phi$, and collect every state that can be reached. Both trajectories are included because the proof compares the two systems at every step. Finally, define the reachable set over the forecasting horizon as
\[
\mathcal H := \bigcup_{t=t_0-1}^{T} \mathcal H_t.
\]

\medskip
\subsection{Forward Invariance}

We next impose a forward-invariance condition on the reachable set. This condition ensures that once the recursion starts inside $\mathcal H$, all subsequent states remain in $\mathcal H$, so the local assumptions used in the proof can be applied recursively over time.

We assume $\mathcal H$ is forward-invariant for both the best-in-class and learned dynamics over $t=t_0-1,\dots,T$ (enforced in practice by bounded GRU activations together with output normalization or clipping). Formally, ``forward-invariant'' means that whenever $\mathbf h\in\mathcal H$, every one-step successor generated by either $f^\star$ (with $\mathbf x\sim p_t^\star(\cdot\mid\mathbf h)$) or $f_\phi$ (with $\widehat{\mathbf x}\sim p_\theta(\cdot\mid\mathbf h)$) also lies in $\mathcal H$; equivalently, once a trajectory enters $\mathcal H$ it cannot leave during the forecasting horizon. This property ensures that all local assumptions stated below can be reused recursively at every time step.

\medskip
\subsection{Moment Bound Assumption}

Finally, we state a moment/support condition on the conditional laws generated on $\mathcal H$. This condition guarantees that the relevant 1-Wasserstein distances are finite and that the expected marginal forecasting error $\Delta_t$ is well defined.

\begin{assumption}[Moment/support control on a forward-invariant reachable set]
\label{assump:moment_support_main}\label{assump:moment_support_app}
Let $\mathcal H$ be the forward-invariant reachable set over $t=t_0-1,\dots,T$. Assume that for every $t\in\{t_0,\dots,T\}$ and every $\mathbf h\in\mathcal H$, both the target and model conditionals belong to $\mathcal P_1(\mathbb R^d)$, and their first moments are uniformly bounded on $\mathcal H$:
\[
\sup_{t_0\le t\le T}\ \sup_{\mathbf h\in\mathcal H}\
\mathbb E_{\mathbf X\sim p_t^\star(\cdot\mid \mathbf h)}\big[\|\mathbf X\|\big]\ \le\ M_1\ <\ \infty,
\qquad
\sup_{t_0\le t\le T}\ \sup_{\mathbf h\in\mathcal H}\
\mathbb E_{\widehat{\mathbf X}\sim p_\theta(\cdot\mid \mathbf h)}\big[\|\widehat{\mathbf X}\|\big]\ \le\ \widehat M_1\ <\ \infty.
\]
Consequently, for any random latent states $(\widehat{\mathbf h}_{t-1},\mathbf h_{t-1}^\star)\in\mathcal H\times\mathcal H$, the conditional Wasserstein distance
\[
W_1\!\big(p_\theta(\cdot\mid \widehat{\mathbf h}_{t-1}),\,p_t^\star(\cdot\mid \mathbf h_{t-1}^\star)\big)
\]
is well defined and integrable, so $\Delta_t=\mathbb E[W_1(\cdot,\cdot)]$ is finite.

A sufficient, and practically enforceable, condition is compact support after the same normalization or clipping used in training and sampling: there exists $B_x>0$ such that, for all $t\in\{t_0,\dots,T\}$ and all $\mathbf h\in\mathcal H$,
\[
\|\mathbf X\|\le B_x \ \text{a.s. under } p_t^\star(\cdot\mid \mathbf h),
\qquad
\|\widehat{\mathbf X}\|\le B_x \ \text{a.s. under } p_\theta(\cdot\mid \mathbf h).
\]
\end{assumption}

\section{Numerical Illustration of the Stability Margin}
\label{sec:numerical_illustration}

This appendix section complements the theoretical analysis by providing a controlled numerical illustration of the stability margin $\kappa$. While the main text establishes an abstract multi-step Wasserstein error bound, the purpose here is to show how that bound maps into observable forecasting behavior under recursive generation. In particular, we visualize both the geometric damping mechanism in the stable regime and the transition to divergent error growth once the stability margin is violated.

To connect the abstract Wasserstein bounds with practical multi-step forecasting behavior, we design a controlled \emph{oracle} experiment in which the one-step ground-truth conditional law is fully specified and can be sampled exactly. The experiment visualizes (i) the geometric damping mechanism implied by Lemma~\ref{lem:hidden_recursion_main} and (ii) the transition from bounded to divergent error growth governed by the stability margin $\kappa$ in Theorem~\ref{thm:multi_step_consistency_main}.

We use a linear--Gaussian oracle for the one-step conditional law on the reachable set $\mathcal{H}$, and a model conditional law of the same family with shared covariance but a slightly misspecified linear mapping, hence a nonzero intrinsic one-step mismatch consistent with $\varepsilon_{\mathrm{gen}}$. To isolate recursive propagation, we evolve both the oracle and learned latent states through contractive linear recursions with contraction factor $\rho\in(0,1)$ and a fixed input-mixing operator, so that the effective constants $(\rho,L_x,L_P)$ are directly controlled. By varying these constants, equivalently by varying $\kappa$ while keeping the remaining components fixed, we generate stable ($\kappa<1$) and unstable ($\kappa>1$) regimes.

We estimate the outer expectations in $E_t$ and $\Delta_t$ using $R=10^4$ independent rollouts. In each rollout, we iteratively sample $\mathbf{x}_t$ from the oracle conditional, sample $\widehat{\mathbf{x}}_t$ from the model conditional given $\widehat{\mathbf{h}}_{t-1}$, and update $(\mathbf{h}_t^\star,\widehat{\mathbf{h}}_t)$ forward. The inner $W_1$ term is evaluated directly under the Gaussian oracle with shared covariance, so no numerical optimal-transport solver is required, and we average across rollouts to obtain $\widehat{E}_t$ and $\widehat{\Delta}_t$.

\begin{figure}[htbp]
\centering
\subfigure[Latent-state error dynamics $\widehat{E}_t$]{
\includegraphics[width=2.5in]{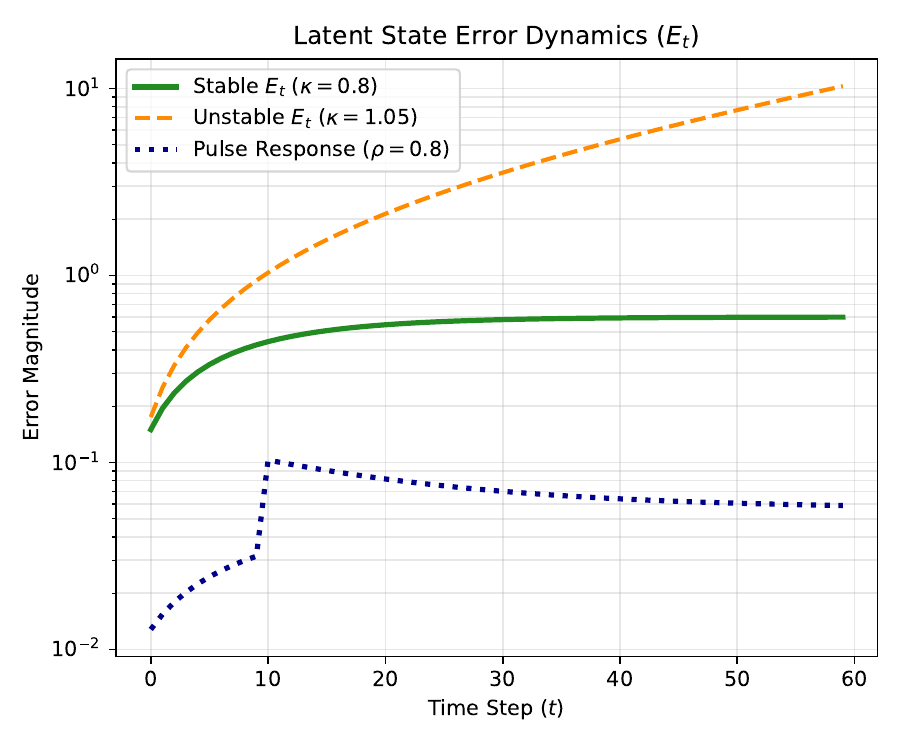}
}\label{fig:latent_error}
\subfigure[Expected one-step discrepancy $\widehat{\Delta}_t$]{
\includegraphics[width=2.5in]{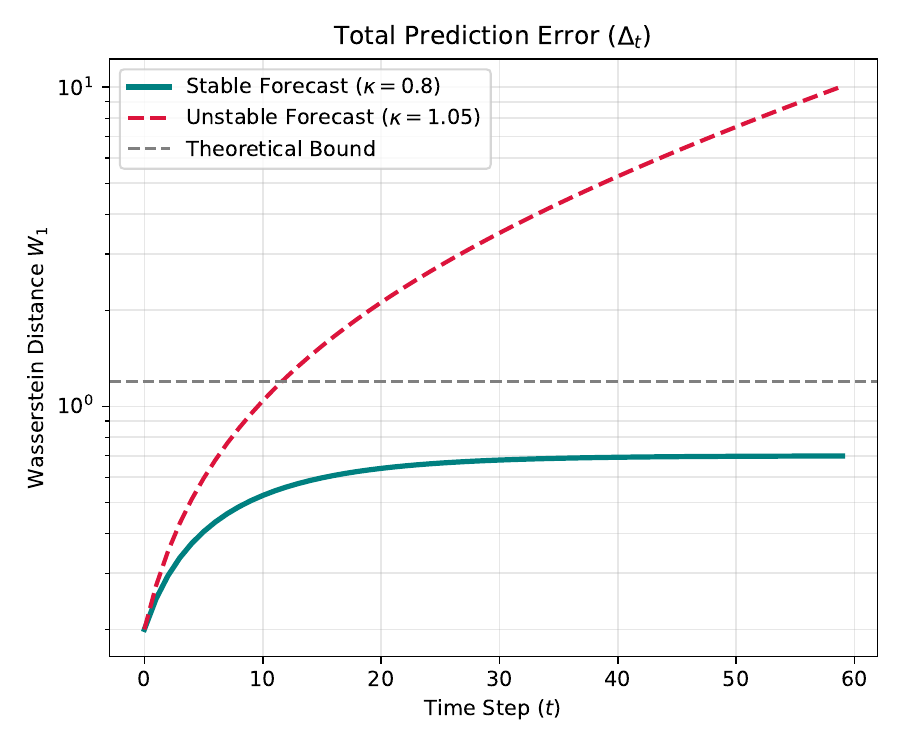}
}\label{fig:total_error}
\caption{Numerical validation of the stability-based error propagation theory (estimated with $R=10^4$ rollouts).
(a) Evolution of the latent-state error $\widehat{E}_t$ under stable vs.\ unstable configurations, including the damping of an injected pulse at $t=10$, consistent with Lemma~\ref{lem:hidden_recursion_main}.
(b) Evolution of the expected one-step discrepancy $\widehat{\Delta}_t$, showing bounded behavior in the stable regime ($\kappa<1$) and divergence in the unstable regime ($\kappa>1$), consistent with Theorem~\ref{thm:multi_step_consistency_main}.}
\label{fig:numerical_validation}
\end{figure}

Figure~\ref{fig:numerical_validation}(a) isolates latent error propagation. In the stable regime (e.g., $\kappa=0.8$), $\widehat{E}_t$ contracts and quickly settles to a small steady level; in the unstable regime (e.g., $\kappa=1.05$), latent mismatch amplifies across steps and grows over time. To make the mechanism explicit, we inject a one-time perturbation to the model latent state at $t=10$ while keeping the oracle sampling mechanism unchanged. Under $\kappa<1$, the excess error decays approximately geometrically, matching the qualitative implication of Lemma~\ref{lem:hidden_recursion_main} that past latent mismatch is discounted by contraction; under $\kappa>1$, the perturbation is amplified by feedback and does not dissipate.

Figure~\ref{fig:numerical_validation}(b) reports the end-to-end expected one-step discrepancy $\widehat{\Delta}_t$. In the stable regime ($\kappa<1$), $\widehat{\Delta}_t$ remains uniformly bounded over the horizon and saturates at a plateau: the intrinsic one-step mismatch (the $\varepsilon_{\mathrm{gen}}$ component) is present, but contraction limits how latent drift transfers into forecast degradation. In contrast, when $\kappa>1$, $\widehat{\Delta}_t$ escalates rapidly, illustrating the divergent cascade emphasized by the theory: one-step forecast mismatch perturbs the latent recursion through $L_x$, the resulting latent drift shifts the target conditional law through $L_P$, and the loop amplifies over time.

Overall, the illustration corroborates that long-horizon robustness depends not only on reducing one-step generative error $\varepsilon_{\mathrm{gen}}$, but also on maintaining a non-explosive stability margin $\kappa<1$ by enforcing hidden-state contraction and reducing effective input sensitivity. A quantitative sensitivity analysis with respect to $\kappa$ is provided in Appendix~\ref{app:sensitivity_kappa}, and a GRU-based derivation translating implementation safeguards into explicit, checkable bounds on $(\rho,L_x)$, and hence on $\kappa$, is provided in Appendix~\ref{app:gru_stability}.

\section{Sensitivity Analysis of the Stability Margin}
\label{app:sensitivity_kappa}

This appendix section complements the numerical illustration above by providing a quantitative sensitivity analysis of the stability margin $\kappa$. While the main theorem establishes the qualitative distinction between the stable regime ($\kappa<1$) and the unstable regime ($\kappa>1$), the purpose here is to characterize how forecasting performance changes continuously as the system approaches the critical threshold.

We perform a parametric sweep by varying $\kappa$ from $0.1$ to $0.99$ while maintaining the remaining system constants. This simulation is designed to validate the inverse relationship between the stability margin and the equilibrium error level, as predicted by the $(1-\kappa)$ term in the denominator of Theorem~\ref{thm:multi_step_consistency_main}.

\begin{figure}[htbp]
     \centering
     \includegraphics[width=0.7\textwidth]{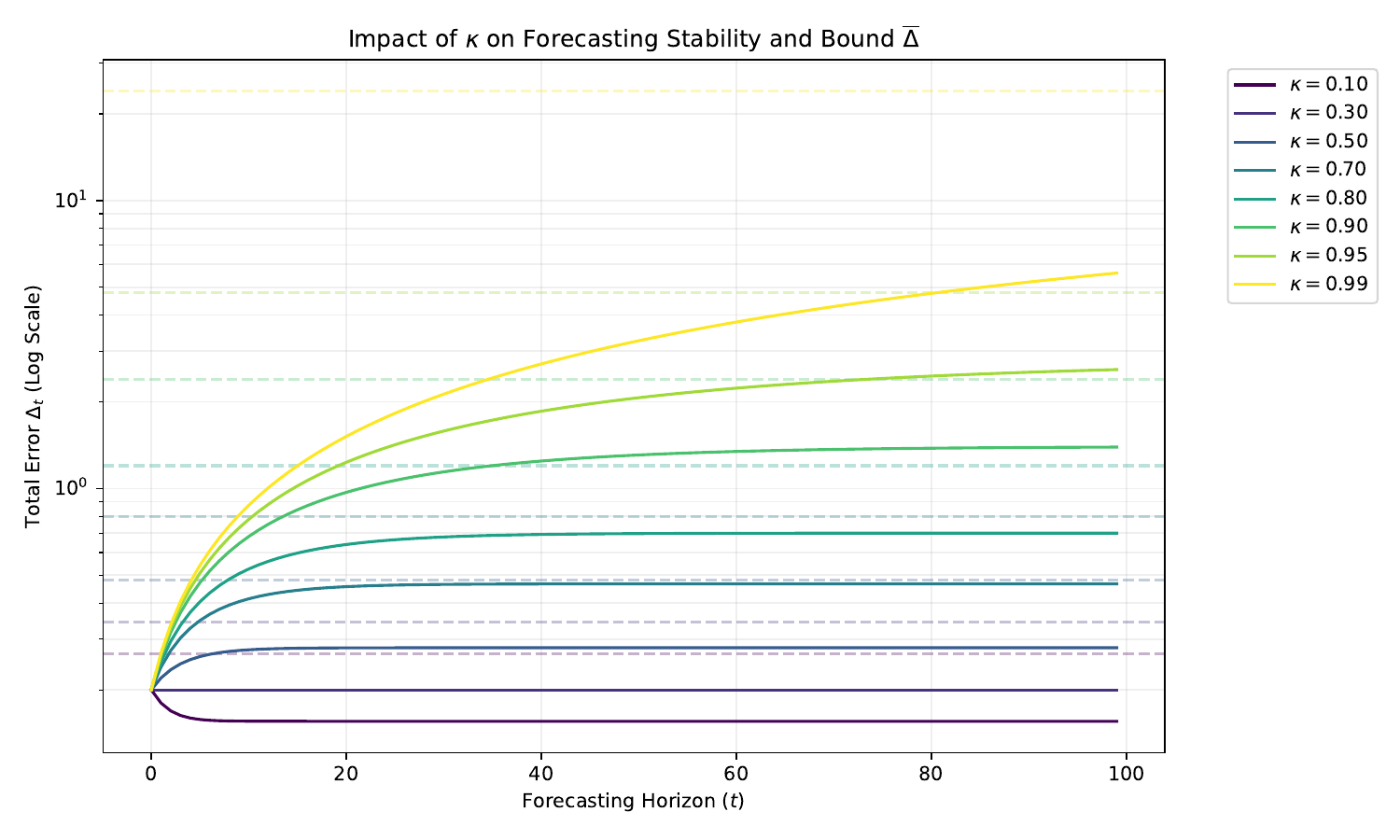}
     \caption{Sensitivity analysis of the forecasting error $\Delta_t$ across a spectrum of stability margins $\kappa$. Solid lines represent empirical error trajectories, while the corresponding dashed lines indicate the theoretical upper bounds $\overline{\Delta}$. As $\kappa$ approaches $1$, the bound exhibits nonlinear inflation.}
     \label{fig:kappa_sensitivity_app}
\end{figure}

The numerical results, visualized in Figure~\ref{fig:kappa_sensitivity_app}, reveal that forecasting quality is highly sensitive to the proximity of $\kappa$ to unity. In the \textit{tight-coupling regime} (e.g., $\kappa \le 0.5$), the total error $\Delta_t$ saturates almost instantaneously, indicating that the contractive power of the recurrent update is sufficient to suppress generation noise before it propagates. Conversely, in the \textit{near-critical regime} (e.g., $\kappa=0.99$), although the system remains mathematically stable, the error plateau inflates by orders of magnitude. In this state, the error trajectory exhibits a prolonged growth regime, signifying that the system has lost most of its inherent self-correction capability.

The results documented here suggest that for high-fidelity autoregressive forecasting, achieving $\kappa < 1$ is a necessary but not sufficient condition for practical reliability. We propose the following engineering insights:
\begin{enumerate}
    \item \textbf{The necessity of a buffer:} To maintain a tight error plateau, the architecture should be regularized to ensure that $\kappa$ remains within a safe range (e.g., $\kappa \in [0.5, 0.8]$). This provides a stability buffer that accounts for potential estimation noise in the score network.
    \item \textbf{Trade-off management:} While a very small $\kappa$ (e.g., $0.1$) yields maximum stability, it may over-constrain the recurrent unit's capacity. The use of a powerful conditional diffusion head allows for stricter regularization of the GRU weights without sacrificing the overall expressive power of the generative process.
\end{enumerate}
Through this sensitivity analysis, we confirm that the theoretical bound derived in Theorem~\ref{thm:multi_step_consistency_main} is not only a qualitative guarantee but also a quantitatively informative indicator of the system's long-run performance.

\section{Checkable Stability for GRU via Jacobian Spectral-Norm Bounds}
\label{app:gru_stability}

This appendix section makes the stability condition $\kappa<1$ in Theorem~\ref{thm:multi_step_consistency_main} practically checkable for the GRU latent update. The main text uses the abstract constants $(\rho,L_x)$ on the forward-invariant reachable set $\mathcal{H}$ of the autoregressive forecasting dynamics. Here we upper bound $(\rho,L_x)$ by trajectory-uniform spectral-norm bounds on GRU Jacobians, which can in turn be controlled through the spectral norms of the GRU weight matrices together with simple gate-range conditions. The final output is a fully checkable sufficient condition that can be enforced during training.

Throughout this appendix, $\|\cdot\|_2$ denotes the matrix spectral norm. Recall that the recurrent update is applied repeatedly over the forecasting horizon $t=t_0,\dots,T$ with time-shared parameters. Since the conditioning context $\mathbf{c}$ is fixed over the horizon, any affine dependence on $\mathbf{c}$ can be absorbed into bias terms; equivalently, one may view $\mathbf{c}$ as part of an augmented input that is constant in time. All bounds below are intended to hold uniformly on a compact, convex set $\mathcal{T}\subset\mathbb{R}^m\times\mathbb{R}^d$ that contains all relevant pairs $(\mathbf{h},\mathbf{x})$ visited by $\{(\widehat{\mathbf{h}}_{t-1},\widehat{\mathbf{x}}_t)\}_{t=t_0}^T$ and by the best-in-class counterpart $\{(\mathbf{h}_{t-1}^\star,\mathbf{x}_t)\}_{t=t_0}^T$ under the fixed context $\mathbf{c}$. In particular, $\mathcal{T}$ can be chosen as a convex superset of the forward-invariant set $\mathcal{H}$ paired with the relevant input range. All vectors and matrices are typeset in boldface.

We first record a standard implication that connects Jacobian spectral-norm bounds to the constants in Assumption~\ref{assump:combined_main}.

\begin{lemma}[Jacobian bounds imply a contractive-Lipschitz inequality on $\mathcal{T}$]
\label{lem:jacobian_to_lipschitz_app}
Assume $f_\phi(\mathbf{h},\mathbf{x},\mathbf{c})$ is continuously differentiable in $(\mathbf{h},\mathbf{x})$ on $\mathcal{T}$.
Define
\[
\rho_\theta \;:=\; \sup_{(\mathbf{h},\mathbf{x})\in\mathcal{T}}
\left\|\frac{\partial f_\phi(\mathbf{h},\mathbf{x},\mathbf{c})}{\partial \mathbf{h}}\right\|_2,
\qquad
L_{x,\theta} \;:=\; \sup_{(\mathbf{h},\mathbf{x})\in\mathcal{T}}
\left\|\frac{\partial f_\phi(\mathbf{h},\mathbf{x},\mathbf{c})}{\partial \mathbf{x}}\right\|_2.
\]
Then for any $(\mathbf{h},\mathbf{x}),(\mathbf{h}',\mathbf{x}')\in\mathcal{T}$,
\[
\|f_\phi(\mathbf{h},\mathbf{x},\mathbf{c})-f_\phi(\mathbf{h}',\mathbf{x}',\mathbf{c})\|
\le
\rho_\theta \|\mathbf{h}-\mathbf{h}'\| + L_{x,\theta}\|\mathbf{x}-\mathbf{x}'\|.
\]
\end{lemma}

\textit{Proof.}\quad
For any $(\mathbf{h},\mathbf{x}),(\mathbf{h}',\mathbf{x}')\in\mathcal{T}$, write
\[
f_\phi(\mathbf{h},\mathbf{x},\mathbf{c})-f_\phi(\mathbf{h}',\mathbf{x}',\mathbf{c})
=
\big(f_\phi(\mathbf{h},\mathbf{x},\mathbf{c})-f_\phi(\mathbf{h}',\mathbf{x},\mathbf{c})\big)
+
\big(f_\phi(\mathbf{h}',\mathbf{x},\mathbf{c})-f_\phi(\mathbf{h}',\mathbf{x}',\mathbf{c})\big).
\]
Consider the first difference with $\mathbf{x}$ fixed. Define the path
$\mathbf{h}(s):=\mathbf{h}'+s(\mathbf{h}-\mathbf{h}')$ for $s\in[0,1]$.
Because $\mathcal{T}$ is convex, $(\mathbf{h}(s),\mathbf{x})\in\mathcal{T}$ for all $s$.
Let $g(s):=f_\phi(\mathbf{h}(s),\mathbf{x},\mathbf{c})$. By the fundamental theorem of calculus,
\[
f_\phi(\mathbf{h},\mathbf{x},\mathbf{c})-f_\phi(\mathbf{h}',\mathbf{x},\mathbf{c})
=
\int_0^1 g'(s)\,ds
=
\int_0^1
\frac{\partial f_\phi(\mathbf{h}(s),\mathbf{x},\mathbf{c})}{\partial \mathbf{h}}(\mathbf{h}-\mathbf{h}')\,ds,
\]
hence
\[
\|f_\phi(\mathbf{h},\mathbf{x},\mathbf{c})-f_\phi(\mathbf{h}',\mathbf{x},\mathbf{c})\|
\le
\int_0^1
\left\|\frac{\partial f_\phi(\mathbf{h}(s),\mathbf{x},\mathbf{c})}{\partial \mathbf{h}}\right\|_2 ds\,
\|\mathbf{h}-\mathbf{h}'\|
\le
\rho_\theta \|\mathbf{h}-\mathbf{h}'\|.
\]
An analogous argument with $\mathbf{h}'$ fixed and the path $\mathbf{x}(s):=\mathbf{x}'+s(\mathbf{x}-\mathbf{x}')$ gives
\[
\|f_\phi(\mathbf{h}',\mathbf{x},\mathbf{c})-f_\phi(\mathbf{h}',\mathbf{x}',\mathbf{c})\|
\le
L_{x,\theta}\|\mathbf{x}-\mathbf{x}'\|.
\]
Combining the two bounds proves the claim.
\hfill$\Box$

We consider a standard GRU cell with hidden state $\mathbf{h}\in\mathbb{R}^m$ and input $\mathbf{x}\in\mathbb{R}^d$. Suppressing the fixed context $\mathbf{c}$ for clarity, the GRU update map in \eqref{eq:ar_gen_rule} is
\[
\mathbf{z}=\sigma(\mathbf{W}_z \mathbf{x} + \mathbf{U}_z \mathbf{h} + \mathbf{b}_z),
\qquad
\mathbf{r}=\sigma(\mathbf{W}_r \mathbf{x} + \mathbf{U}_r \mathbf{h} + \mathbf{b}_r),
\]
\[
\widetilde{\mathbf{h}}=\tanh\!\big(\mathbf{W}_h \mathbf{x} + \mathbf{U}_h(\mathbf{r}\odot \mathbf{h}) + \mathbf{b}_h\big),
\qquad
f_\phi(\mathbf{h},\mathbf{x},\mathbf{c})=(\mathbf{1}-\mathbf{z})\odot \mathbf{h} + \mathbf{z}\odot \widetilde{\mathbf{h}}.
\]
Here $\mathbf{W}_z,\mathbf{W}_r,\mathbf{W}_h\in\mathbb{R}^{m\times d}$ and $\mathbf{U}_z,\mathbf{U}_r,\mathbf{U}_h\in\mathbb{R}^{m\times m}$ are weight matrices, and $\mathbf{b}_z,\mathbf{b}_r,\mathbf{b}_h\in\mathbb{R}^m$ are bias vectors. $\odot$ denotes the Hadamard product.

\begin{assumption}[Trajectory boundedness and gate ranges on $\mathcal{T}$]
\label{assump:gru_gate_bounds_app}
On $\mathcal{T}$, we assume
\[
\|\mathbf{h}\|_\infty \le 1,
\]
and there exist constants $z_{\min}\in(0,1)$, $z_{\max}\in(0,1)$, and $r_{\max}\in(0,1)$ such that
\[
z_{\min}\le z_i \le z_{\max},
\qquad
0\le r_i \le r_{\max},
\qquad
\text{for all } i=1,\dots,m,
\]
where $\mathbf{z}=(z_1,\dots,z_m)^\top$ and $\mathbf{r}=(r_1,\dots,r_m)^\top$ are the update and reset gates.
\end{assumption}

\begin{remark}
Assumption~\ref{assump:gru_gate_bounds_app} restricts the analysis to the compact set $\mathcal{T}$ that contains the actual and oracle trajectories. This is aligned with the main text, where the stability analysis is carried out on the forward-invariant reachable set $\mathcal{H}$ rather than globally on $\mathbb{R}^m$. The bound $\|\mathbf{h}\|_\infty\le 1$ is natural for the GRU because $f_\phi$ is an elementwise convex combination of $\mathbf{h}$ and $\widetilde{\mathbf{h}}\in[-1,1]^m$. The gate-range conditions exclude degenerate regimes in which $\mathbf{z}$ approaches $\mathbf{0}$ or $\mathbf{1}$, which can make uniform contraction bounds either impossible or excessively conservative. These ranges can be encouraged by initialization and mild regularization, and in practice they can be monitored empirically on rollout trajectories. This highlights that a minimum update rate $z_{\min}>0$ acts as a restorative force that prevents the latent state from becoming a pure random walk.
\end{remark}

We repeatedly use the derivative bounds
\[
\sup_u |\tanh'(u)| \le 1,
\qquad
\sup_u |\sigma'(u)| \le \frac{1}{4}.
\]
For any vector $\mathbf{v}$, $\|\mathrm{Diag}(\mathbf{v})\|_2=\|\mathbf{v}\|_\infty$.

\begin{lemma}[GRU Jacobian spectral-norm bounds on $\mathcal{T}$]
\label{lem:gru_jacobian_bounds_app}
Under Assumption~\ref{assump:gru_gate_bounds_app}, the GRU update map satisfies
\begin{equation}
\label{eq:gru_rho_bar_app}
\left\|\frac{\partial f_\phi(\mathbf{h},\mathbf{x},\mathbf{c})}{\partial \mathbf{h}}\right\|_2
\le
(1-z_{\min})
+\frac{1}{2}\|\mathbf{U}_z\|_2
+ z_{\max}\,\|\mathbf{U}_h\|_2\!\left(r_{\max}+\frac{1}{4}\|\mathbf{U}_r\|_2\right),
\end{equation}
and
\begin{equation}
\label{eq:gru_Lx_bar_app}
\left\|\frac{\partial f_\phi(\mathbf{h},\mathbf{x},\mathbf{c})}{\partial \mathbf{x}}\right\|_2
\le
\frac{1}{2}\|\mathbf{W}_z\|_2
+ z_{\max}\left(\|\mathbf{W}_h\|_2+\frac{1}{4}\|\mathbf{U}_h\|_2\,\|\mathbf{W}_r\|_2\right).
\end{equation}
Define the right-hand sides of \eqref{eq:gru_rho_bar_app} and \eqref{eq:gru_Lx_bar_app} as
$\overline{\rho}_{\mathrm{GRU}}$ and $\overline{L}_{x,\mathrm{GRU}}$, respectively. Then, on $\mathcal{T}$,
\[
\rho_\theta \le \overline{\rho}_{\mathrm{GRU}},
\qquad
L_{x,\theta}\le \overline{L}_{x,\mathrm{GRU}}.
\]
\end{lemma}

\begin{remark}
Lemma~\ref{lem:gru_jacobian_bounds_app} is fully explicit and checkable: it depends only on spectral norms of GRU weight matrices and gate-range constants $(z_{\min},z_{\max},r_{\max})$. The bound is conservative but transparent, and it supports stability verification by combining Lemma~\ref{lem:jacobian_to_lipschitz_app} with Theorem~\ref{thm:multi_step_consistency_main}. The Jacobian bounds are sufficient, not necessary, and can be conservative because several correlated terms are upper bounded via triangle inequalities. In practice, stability is commonly enforced by spectral normalization or spectral clipping on GRU weight matrices, which directly controls the spectral norms appearing in the bounds and makes the sufficient condition checkable.
\end{remark}

\textit{Proof.}\quad
From $f_\phi(\mathbf{h},\mathbf{x},\mathbf{c})=(\mathbf{1}-\mathbf{z})\odot \mathbf{h} + \mathbf{z}\odot \widetilde{\mathbf{h}}$, the Jacobian with respect to $\mathbf{h}$ is
\[
\frac{\partial f_\phi}{\partial \mathbf{h}}
=
\mathrm{Diag}(\mathbf{1}-\mathbf{z})
+ \mathrm{Diag}(\mathbf{h}-\widetilde{\mathbf{h}})\frac{\partial \mathbf{z}}{\partial \mathbf{h}}
+ \mathrm{Diag}(\mathbf{z})\frac{\partial \widetilde{\mathbf{h}}}{\partial \mathbf{h}}.
\]
First, $\|\mathrm{Diag}(\mathbf{1}-\mathbf{z})\|_2=\|\mathbf{1}-\mathbf{z}\|_\infty\le 1-z_{\min}$.

Second, $\|\mathbf{h}\|_\infty\le 1$ and $\|\widetilde{\mathbf{h}}\|_\infty\le 1$ imply $\|\mathbf{h}-\widetilde{\mathbf{h}}\|_\infty\le 2$, hence $\|\mathrm{Diag}(\mathbf{h}-\widetilde{\mathbf{h}})\|_2\le 2$. Moreover,
\[
\frac{\partial \mathbf{z}}{\partial \mathbf{h}}
=
\mathrm{Diag}\big(\sigma'(\mathbf{W}_z \mathbf{x}+\mathbf{U}_z \mathbf{h}+\mathbf{b}_z)\big)\,\mathbf{U}_z,
\qquad
\left\|\frac{\partial \mathbf{z}}{\partial \mathbf{h}}\right\|_2\le \frac{1}{4}\|\mathbf{U}_z\|_2.
\]
Thus the middle term is bounded by $2\cdot \frac{1}{4}\|\mathbf{U}_z\|_2=\frac{1}{2}\|\mathbf{U}_z\|_2$.

Third, define $\mathbf{a}_h:=\mathbf{W}_h \mathbf{x} + \mathbf{U}_h(\mathbf{r}\odot \mathbf{h})+\mathbf{b}_h$. Then
\[
\frac{\partial \widetilde{\mathbf{h}}}{\partial \mathbf{h}}
=
\mathrm{Diag}\big(\tanh'(\mathbf{a}_h)\big)\,\mathbf{U}_h\left(\mathrm{Diag}(\mathbf{r})+\mathrm{Diag}(\mathbf{h})\frac{\partial \mathbf{r}}{\partial \mathbf{h}}\right).
\]
Using $\|\mathrm{Diag}(\tanh'(\mathbf{a}_h))\|_2\le 1$, $\|\mathrm{Diag}(\mathbf{r})\|_2\le r_{\max}$, $\|\mathrm{Diag}(\mathbf{h})\|_2\le 1$, and
\[
\frac{\partial \mathbf{r}}{\partial \mathbf{h}}
=
\mathrm{Diag}\big(\sigma'(\mathbf{W}_r \mathbf{x}+\mathbf{U}_r \mathbf{h}+\mathbf{b}_r)\big)\,\mathbf{U}_r,
\qquad
\left\|\frac{\partial \mathbf{r}}{\partial \mathbf{h}}\right\|_2\le \frac{1}{4}\|\mathbf{U}_r\|_2,
\]
we obtain
\[
\left\|\frac{\partial \widetilde{\mathbf{h}}}{\partial \mathbf{h}}\right\|_2
\le
\|\mathbf{U}_h\|_2\left(r_{\max}+\frac{1}{4}\|\mathbf{U}_r\|_2\right).
\]
Finally, $\|\mathrm{Diag}(\mathbf{z})\|_2\le z_{\max}$, hence
\[
\left\|\mathrm{Diag}(\mathbf{z})\frac{\partial \widetilde{\mathbf{h}}}{\partial \mathbf{h}}\right\|_2
\le
z_{\max}\,\|\mathbf{U}_h\|_2\left(r_{\max}+\frac{1}{4}\|\mathbf{U}_r\|_2\right).
\]
Summing the three bounds yields \eqref{eq:gru_rho_bar_app}.

For the Jacobian with respect to $\mathbf{x}$,
\[
\frac{\partial f_\phi}{\partial \mathbf{x}}
=
\mathrm{Diag}(\mathbf{h}-\widetilde{\mathbf{h}})\frac{\partial \mathbf{z}}{\partial \mathbf{x}}
+\mathrm{Diag}(\mathbf{z})\frac{\partial \widetilde{\mathbf{h}}}{\partial \mathbf{x}}.
\]
As above, $\|\mathrm{Diag}(\mathbf{h}-\widetilde{\mathbf{h}})\|_2\le 2$ and
\[
\frac{\partial \mathbf{z}}{\partial \mathbf{x}}
=
\mathrm{Diag}\big(\sigma'(\mathbf{W}_z \mathbf{x}+\mathbf{U}_z \mathbf{h}+\mathbf{b}_z)\big)\,\mathbf{W}_z,
\qquad
\left\|\frac{\partial \mathbf{z}}{\partial \mathbf{x}}\right\|_2\le \frac{1}{4}\|\mathbf{W}_z\|_2,
\]
giving $\frac{1}{2}\|\mathbf{W}_z\|_2$ for the first term.
For the second term,
\[
\frac{\partial \widetilde{\mathbf{h}}}{\partial \mathbf{x}}
=
\mathrm{Diag}\big(\tanh'(\mathbf{a}_h)\big)\left(\mathbf{W}_h + \mathbf{U}_h\,\mathrm{Diag}(\mathbf{h})\frac{\partial \mathbf{r}}{\partial \mathbf{x}}\right),
\]
and
\[
\frac{\partial \mathbf{r}}{\partial \mathbf{x}}
=
\mathrm{Diag}\big(\sigma'(\mathbf{W}_r \mathbf{x}+\mathbf{U}_r \mathbf{h}+\mathbf{b}_r)\big)\,\mathbf{W}_r,
\qquad
\left\|\frac{\partial \mathbf{r}}{\partial \mathbf{x}}\right\|_2\le \frac{1}{4}\|\mathbf{W}_r\|_2.
\]
Thus,
\[
\left\|\frac{\partial \widetilde{\mathbf{h}}}{\partial \mathbf{x}}\right\|_2
\le \|\mathbf{W}_h\|_2 + \frac{1}{4}\|\mathbf{U}_h\|_2\|\mathbf{W}_r\|_2,
\]
and multiplying by $\|\mathrm{Diag}(\mathbf{z})\|_2\le z_{\max}$ yields \eqref{eq:gru_Lx_bar_app}.
\hfill$\Box$

Combining Lemma~\ref{lem:gru_jacobian_bounds_app} with Lemma~\ref{lem:jacobian_to_lipschitz_app}, one may take
\[
\rho := \rho_\theta,\quad L_x := L_{x,\theta}
\qquad\text{and in particular}\qquad
\rho \le \overline{\rho}_{\mathrm{GRU}},\quad L_x\le \overline{L}_{x,\mathrm{GRU}}.
\]
Therefore, a fully checkable sufficient condition that implies the stability margin $\kappa<1$ in Theorem~\ref{thm:multi_step_consistency_main} is
\[
\overline{\rho}_{\mathrm{GRU}}<1
\qquad\text{and}\qquad
L_P\,\overline{L}_{x,\mathrm{GRU}} < 1-\overline{\rho}_{\mathrm{GRU}}.
\]
This condition can be enforced by spectral normalization or spectral clipping on the GRU weight matrices $\mathbf{U}_z,\mathbf{U}_r,\mathbf{U}_h,\mathbf{W}_z,\mathbf{W}_r,\mathbf{W}_h$, together with mild regularization that keeps the empirical gate ranges within $(z_{\min},z_{\max},r_{\max})$ on rollout trajectories.
\end{document}